\documentclass{article}

\usepackage{microtype}
\usepackage{graphicx}
\usepackage{subcaption}
\usepackage{booktabs} % for professional tables

\usepackage{hyperref}

\usepackage[preprint]{icml2026}

\usepackage{amsmath}
\usepackage{amssymb}
\usepackage{amsthm}
\usepackage{mathtools}
\usepackage{enumitem}
\usepackage{multirow}
\usepackage{colortbl}
\usepackage{xcolor}
\usepackage{tikz}
\usetikzlibrary{shadows}

\theoremstyle{plain}
\newtheorem{theorem}{Theorem}[section] 
\newtheorem{lemma}[theorem]{Lemma}
\newtheorem{proposition}[theorem]{Proposition}
\newtheorem{corollary}[theorem]{Corollary}

\theoremstyle{definition}

\theoremstyle{remark}

\usepackage{booktabs}
\usepackage{multirow}
\usepackage{tabularx}
\usepackage{array}
\usepackage{makecell}
\usepackage[dvipsnames]{xcolor}
\usepackage{colortbl}

\usepackage{subcaption} 

\usepackage[most]{tcolorbox}
\newtcbox{\eqhlR}{on line, arc=2pt, boxsep=1pt, colback=red!10, colframe=red!10}
\newtcbox{\eqhlB}{on line, arc=2pt, boxsep=1pt, colback=blue!10, colframe=blue!10}

\usepackage{hyperref}

\definecolor{UpGreen}{RGB}{34, 139, 34}
\definecolor{DownOrange}{RGB}{210, 105, 30}
\definecolor{NeutralGray}{RGB}{120, 120, 120}
\newcommand{\inc}[1]{\textcolor{UpGreen}{\scriptsize\,\raisebox{0.2ex}{$\uparrow$}\,#1}}
\newcommand{\dec}[1]{\textcolor{DownOrange}{\scriptsize\,\raisebox{0.2ex}{$\downarrow$}\,#1}}

\newcommand{\n}{\textsc{BiCC}} % 
\newcommand{\m}{\textsc{RCC}} % 

\definecolor{rowblue}{RGB}{235, 242, 250}

\newcommand{\oursrow}{\rowcolor{rowblue}}

\usepackage[capitalize,noabbrev]{cleveref}

\usepackage[most]{tcolorbox}
\usepackage{fontawesome5}

\definecolor{insightblue}{RGB}{31, 78, 121}

\newtcolorbox{keyinsight}{
  colback=insightblue!4!white,
  colframe=insightblue,
  fonttitle=\bfseries\sffamily,
  title={\raisebox{-0.15em}{\faLightbulb[regular]}\ Insights},
  enhanced,
  attach boxed title to top left={yshift=-2.5mm, xshift=5mm},
  boxed title style={colback=insightblue, colframe=insightblue, sharp corners},
  left=4mm, right=4mm, top=4mm, bottom=4mm,
  arc=1.5mm, boxrule=1.2pt,
  drop shadow={shadow xshift=0.5mm, shadow yshift=-0.5mm, opacity=0.15},
}

\usepackage[textsize=tiny]{todonotes}

\icmltitlerunning{ICML 2026 Submission}

\begin{document}

\twocolumn[
  \icmltitle{When Right Meets Wrong: Bilateral Context Conditioning with Reward-Confidence Correction for GRPO}

  % It is OKAY to include author information, even for blind submissions: the
  % style file will automatically remove it for you unless you've provided
  % the [accepted] option to the icml2026 package.

  % List of affiliations: The first argument should be a (short) identifier you
  % will use later to specify author affiliations Academic affiliations
  % should list Department, University, City, Region, Country Industry
  % affiliations should list Company, City, Region, Country

  % You can specify symbols, otherwise they are numbered in order. Ideally, you
  % should not use this facility. Affiliations will be numbered in order of
  % appearance and this is the preferred way.
  \icmlsetsymbol{equal}{*}

  \begin{icmlauthorlist}
    \icmlauthor{Yu Li}{ece}
    \icmlauthor{Tian Lan}{ece}
    \icmlauthor{Zhengling Qi}{bus}
    \end{icmlauthorlist}
    \icmlaffiliation{ece}{Department of Electrical and Computer Engineering, George Washington University}
    \icmlaffiliation{bus}{School of Business, George Washington University}
    \icmlcorrespondingauthor{Tian Lan}{tlan@gwu.edu}
    \icmlcorrespondingauthor{Zhengling Qi}{qizhengling@gwu.edu}

  % You may provide any keywords that you find helpful for describing your
  % paper; these are used to populate the "keywords" metadata in the PDF but
  % will not be shown in the document
  \icmlkeywords{Machine Learning, ICML}

  \vskip 0.3in
]

% Use ONE of the following lines. DO NOT remove the command.
% If you have no special notice, KEEP empty braces:
\printAffiliationsAndNotice{}  % no special notice (required even if empty)
% Or, if applicable, use the standard equal contribution text:
% \printAffiliationsAndNotice{\icmlEqualContribution}

\begin{abstract}
Group Relative Policy Optimization (GRPO) has emerged as an effective method for training reasoning models. While it computes advantages based on group mean, GRPO treats each output as an independent sample during the optimization and overlooks a vital structural signal: the natural contrast between correct and incorrect solutions within the same group, thus ignoring the rich, comparative data that could be leveraged by explicitly pitting successful reasoning traces against failed ones. To capitalize on this,
we present a contrastive reformulation of GRPO, showing that the GRPO objective implicitly maximizes the margin between the policy ratios of correct and incorrect samples. Building on this insight, we propose Bilateral Context Conditioning (\n{}), a mechanism that allows the model to cross-reference successful and failed reasoning traces during the optimization, enabling a direct information flow across samples.
We further introduce Reward-Confidence Correction (\m{}) to stabilize training by dynamically adjusts the advantage baseline in GRPO using reward-confidence covariance derived from the first-order approximation of the variance-minimizing estimator.
Both mechanisms require no additional sampling or auxiliary models and can be adapted to all GRPO variants.
Experiments on mathematical reasoning benchmarks demonstrate consistent improvements across comprehensive models and algorithms.
Code is available at \href{https://github.com/Skylanding/BiCC}{https://github.com/Skylanding/BiCC}.
\end{abstract}
\section{Introduction}
\label{sec:intro}
Large language models (LLMs) have demonstrated remarkable reasoning capabilities through Reinforcement Learning with Verifiable Rewards (RLVR), where models learn from outcome-based feedback such as binary signals of correctness on the mathematical problem solving~\citep{lightman2023let,shao2024deepseekmath,guo2025deepseek}.
Among various policy optimization approaches, Group Relative Policy Optimization (GRPO) has emerged as a prevalent method due to its simplicity: rather than learning a separate critic as in Proximal Policy Optimization \citep[PPO;][]{schulman2017proximal},  GRPO samples multiple candidate solutions for each query and uses their relative performance to estimate advantages~\citep{guo2025deepseek}. 
The group-based design has proven highly effective for training reasoning models and has been adopted in several state-of-the-art methods~\citep[e.g.,][]{mroueh2025revisiting,yang2025qwen3}.

However, vanilla GRPO misses on an important opportunity. Consider a group eight solution samples partitioned into correct and incorrect subsets. The two subsets
%Consider a scenario where the model samples eight solutions to a math problem, three correct and five incorrect. Within the sampled group, a natural structure emerges: outputs partition into correct and incorrect subsets based on verifiable rewards.The two partitions 
often exhibit distinct reasoning patterns, as correct solutions tend to share successful strategies while incorrect ones reveal common failure modes~\citep{uesato2022solving,tyen2024llms}.
Yet GRPO fails to exploit this partitioned structure~\citep{yue2025vapo}. 
By evaluating each solution in isolation with respect to only group mean, the algorithm remains 'blind' to the rest of the group except for calculating advantages.
Consequently, contrastive signals that could be derived from comparing successful and failed paths within the same context has not been leveraged

To this end, we formalize the above observation by first mathematically reformulating the GRPO objective into a contrastive form.
%\begin{keyinsight}
%Within each sampled group, outputs naturally partition into correct ($\mathcal{O}^+$) and incorrect ($\mathcal{O}^-$) subsets based on verifiable rewards, yet the two partitions do not interact during training.
%\end{keyinsight}
The reformulation reveals that GRPO implicitly maximizes the margin between average policy ratios of correct and incorrect samples. Such reformulation makes the important partition structure behind GRPO visible and exploitable. Building on the insight, we propose \textbf{Bilateral Context Conditioning (\n{})}, which enables explicit cross-partition information flow. 
The key idea is to let the model observe incorrect attempts when evaluating correct ones, and vice versa, strengthening the signal within the same group to update the policy.
Grounded in Learning Using Privileged Information~\citep[LUPI;][]{pechyony2010theory}, the opposite-partition samples serve as privileged context available only during training, enabling contrastive learning with zero inference overhead.

To stabilize training under bilateral conditioning, we further introduce \textbf{Reward-Confidence Correction (\m{})}.
Through the first-order approximation of the variance-minimizing baseline under importance sampling method, we derive a covariance-based correction term related to reward to further reduce the variance in GRPO. 
%This term 
It accounts for the correlation between the model's output confidence and the resulting reward, effectively reducing gradient variance without using additional sampling or auxiliary models.

The proposed two mechanisms operate on complementary components of the GRPO and its variants: \n{} modifies the policy ratio by incorporating cross-partition context, while \m{} refines the advantage estimation through baseline correction.
We evaluate on Qwen3-4B and Phi-4-mini across four mathematical reasoning benchmarks: Math500, AMC 2023, AIME 2024, and AIME 2025. 
\n{} yields consistent gain of 0.3--1.9 percentage points across settings, with larger gains on weaker base models. 
\m{} further stabilizes training by reducing gradient variance by 25--35\%.

Our contributions are as follows:
\begin{itemize}[leftmargin=*, nosep]
    \item A contrastive reformulation of GRPO that exposes implicit partition structure within sampled groups, revealing that the objective maximizes the margin between policy ratios of correct and incorrect samples.
    \item Bilateral Context Conditioning (\n{}), which enables right and wrong attempts to inform each other through cross-partition context augmentation, grounded in LUPI.
    \item Reward-Confidence Correction (\m{}), derived from first-order approximation of the optimal baseline, which reduces gradient variances and stabilizes \n{} training by incorporating reward-confidence covariance.
    \item Empirical validation across GRPO and its variants on two base models, demonstrating consistent gains on competition-level mathematical reasoning benchmarks.
\end{itemize}
\section{Preliminary}
\label{sec:background}

\subsection{Reinforcement Learning for LLM reasoning}
\label{sec:llm_rl}
Consider a query $q$ drawn from a task distribution $\mathcal{D}$ and an output sequence $o = (o_1, \ldots, o_T)$ with $T$ tokens. The policy $\pi_\theta$, parameterized by LLM weights $\theta$, defines a distribution over outputs given $q$ that factorizes autoregressively:
\begin{equation}
\pi_\theta(o|q) = \prod_{t=1}^{T} \pi_\theta(o_t|q, o_{<t})
\end{equation}
where $o_{<t} = (o_1, \ldots, o_{t-1})$ denotes the prefix. For reasoning tasks with verifiable answers, the reward function $R(o, q) \in \{0, 1\}$ is binary, where $1$ is referred to as correct and $0$ otherwise. The objective for training a reasoning model is to maximize expected reward over the policy weight $\theta$:
\begin{equation}\label{eqn: obj}
J(\theta) = \mathbb{E}_{q \sim \mathcal{D},\, o \sim \pi_\theta(\cdot|q)}[R(o, q)]
\end{equation}

\textbf{PPO.} 
Proximal Policy Optimization~\citep{schulman2017proximal} updates the policy using importance sampling from a previous iterate $\pi_{\theta_{\text{old}}}$. At state $s_t = (q, o_{<t})$ with action $a_t = o_t$, the policy importance ratio is $\rho_t = \pi_\theta(a_t|s_t) / \pi_{\theta_{\text{old}}}(a_t|s_t)$. The clipped surrogate objective as an approximation to \eqref{eqn: obj} is defined as
\begin{equation}
L^{\text{CLIP}}(\theta) = \mathbb{E}\bigl[\min\bigl(\rho_t A_t, \text{clip}(\rho_t, 1{-}\epsilon, 1{+}\epsilon) A_t\bigr)\bigr],
\end{equation}
where $\epsilon$ is the clipping threshold. The advantage $A_t = Q(s_t, a_t) - V(s_t)$ measures how much better action $a_t$ is compared to average, with $Q$ and $V$ denoting action-value and state-value functions estimated by a learned critic.

\textbf{GRPO.} 
Group Relative Policy Optimization~\citep{guo2025deepseek} eliminates the critic model by estimating advantages from grouped samples. For each query $q$, GRPO samples $G$ outputs $\{o_1, \ldots, o_G\}$ from $\pi_{\theta_{\text{old}}}$, where $o_i = (o_{i,1}, \ldots, o_{i,T_i})$ has $T_i$ tokens and receives reward $r_i = R(o_i, q)$. The token-level objective of GRPO is defined as
\begin{equation}
J_{\text{GRPO}}(\theta) = \mathbb{E}_{q \sim \mathcal{D}} \left[ 
  \frac{1}{G} \sum_{i=1}^{G} \frac{1}{T_i} \sum_{t=1}^{T_i} L_{i,t}^{\text{CLIP}}
\right]
\label{eq:grpo}
\end{equation}
where $L_{i,t}^{\text{CLIP}} = \min\bigl(\rho_{i,t} A_i, \text{clip}(\rho_{i,t}, 1{-}\epsilon, 1{+}\epsilon) A_i\bigr)$ and $\rho_{i,t} = \pi_\theta(o_{i,t}|q, o_{i,<t}) / \pi_{\theta_{\text{old}}}(o_{i,t}|q, o_{i,<t})$. The advantage $A_i$ is constant across all tokens within $o_i$ and computed via group-wise standardization. Specifically,
\begin{equation}
A_i = \frac{r_i - \mu}{\sigma}, \quad 
\mu = \frac{1}{G}\sum_{j=1}^{G} r_j, \quad
\sigma = \sqrt{\frac{1}{G}\sum_{j=1}^{G}(r_j - \mu)^2}
\label{eq:advantage}
\end{equation}
The group mean $\mu$ replaces the learned critic in PPO as a baseline, while division by $\sigma$ normalizes the advantage scale across queries due to varying difficulty.

\subsection{Policy Gradient with Variance Reduction}
\label{sec:pg_baseline}

Once the objective function is constructed, policy gradient methods are typically used. The policy gradient theorem~\citep{williams1992simple,sutton1999policy} expresses the gradient of expected reward as:
\begin{align}
\nabla_\theta J(\theta) &= \mathbb{E}_{\pi_\theta}\left[R(o, q) \nabla_\theta \log \pi_\theta(o|q)\right]\\
&= \mathbb{E}\left[\frac{\pi_\theta(o | q)}{\pi_{\text{ref}}(o | q)}R(o, q) \nabla_\theta \log \pi_\theta(o|q)\right].
\end{align}
It is known that subtracting a baseline function $b$ properly can reduce variance without changing the gradient, i.e.,
\begin{equation}
\nabla_\theta J = \mathbb{E}\left[\frac{\pi_\theta}{\pi_{\text{ref}}}(R - b) \nabla_\theta \log \pi_\theta\right]
\label{eq:pg_baseline}
\end{equation}
Under importance sampling with weight $w = \pi_\theta / \pi_{\text{ref}}$, minimizing the variance of the objective function with respect to $b$ yields the optimal baseline: 
\begin{equation}
b^* = \frac{\mathbb{E}[R w^2]}{\mathbb{E}[w^2]}.
\label{eq:optimal_baseline}
\end{equation}
See the derivation in Appendix \ref{sec: optimal baseline derivation}.
When $w \equiv 1$ (on-policy), it reduces to $b^* = \mathbb{E}[R]$, which is used in  GRPO objective. The group mean is optimal only when importance weights are independent of rewards, which is often not the case \citep{kakade2002approximately,schulman2015trust}.

\begin{figure*}[!t]
    \centering
    \includegraphics[width=\textwidth]{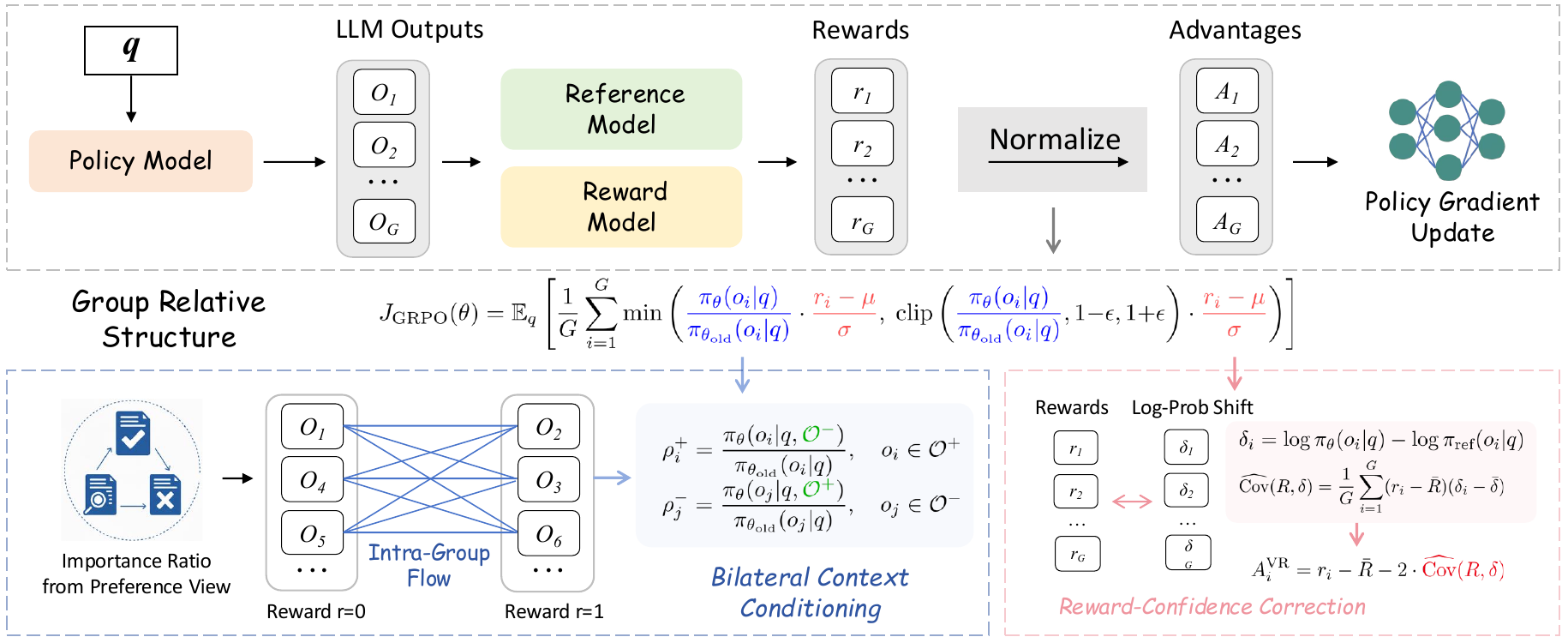}
    \caption{Overview of the proposed methods. The standard GRPO pipeline processes query $q$ through the policy model to generate $G$ outputs, which are scored by reward model and normalized to compute advantages for policy gradient update. \n{} partitions outputs by reward and conditions each sample on opposite-partition outputs, enabling cross-partition information flow where right and wrong attempts inform each other. \m{} computes covariance between rewards and log-probability shifts to correct the advantage estimation.}
    \label{fig:overview}
\end{figure*}

\section{Another Look at GRPO and Variance Reduction}
\label{sec:analysis}

\subsection{Contrastive Reformulation of GRPO Objective}
\label{sec:contrastive}
The rollout within each group in GRPO produces a rich structure: for each query, correct and incorrect outputs naturally form two partitions with distinct optimization roles~\citep{mroueh2025revisiting,mroueh2025reinforcement,li2025inspo}.
We formalize this observation by deriving an equivalent form of GRPO objective under binary rewards $r_i$.

\textbf{Sample Partition.}
For each query $x$, we can naturally partition the group of outputs  into two disjoint subsets based on their corresponding binary rewards $\{r_i\}_{i=1}^G$. Consider
\begin{equation}
\mathcal{O}^+ = \{o_i \mid r_i = 1\}, \quad \mathcal{O}^- = \{o_i \mid r_i = 0\},
\end{equation}
where $G^+ = |\mathcal{O}^+|$ and $G^- = |\mathcal{O}^-|$ represent the cardinality of the correct and incorrect subsets, respectively, and $\hat{p} = G^+/G$ denotes the sample percentage of positive rewards in the group. 
Thus $\mu = \hat{p}$ and $\sigma^2 = \hat{p}(1-\hat{p})$.
Substituting into Eq.~\eqref{eq:advantage} gives advantages in each group as
\begin{equation}
A^+ = \sqrt{\frac{1-\hat{p}}{\hat{p}}}, \quad A^- = -\sqrt{\frac{\hat{p}}{1-\hat{p}}}
\label{eq:binary_advantage}
\end{equation}
respectively. Note that $A^+ > 0$ and $A^- < 0$.
The clipping mechanism behaves asymmetrically depending on the sign of the advantage. In particular,
when $A > 0$, the $\min$ in the clipped objective selects the smaller positive value, yielding $A \cdot \min(\rho, 1+\epsilon)$. 
When $A < 0$, both terms are negative, so the $\min$ selects the more negative value, yielding $A \cdot \max(\rho, 1-\epsilon)$. 
We therefore define:
\begin{equation}
\mathcal{C}_{\text{up}}(\rho) = \min(\rho, 1+\epsilon), \quad \mathcal{C}_{\text{low}}(\rho) = \max(\rho, 1-\epsilon).
\label{eq:clip_functions}
\end{equation}

\textbf{Contrastive Form.} 
Applying $\mathcal{C}_{\text{up}}$ to positive samples and $\mathcal{C}_{\text{low}}$ to negative samples, the GRPO objective can be written as:
\begin{equation}
J_{\text{GRPO}} = \mathbb{E}_{q} \left[ 
  \frac{G^+}{G} A^+ \bar{\rho}^+_{\text{clip}}
  + \frac{G^-}{G} A^- \bar{\rho}^-_{\text{clip}}
\right]
\end{equation}
where $\bar{\rho}^+_{\text{clip}} = \frac{1}{G^+} \sum_{i \in \mathcal{O}^+} \mathcal{C}_{\text{up}}(\rho_i)$ and $\bar{\rho}^-_{\text{clip}} = \frac{1}{G^-} \sum_{j \in \mathcal{O}^-} \mathcal{C}_{\text{low}}(\rho_j)$ are the average clipped ratios. Since $\hat{p} \cdot A^+ = (1-\hat{p}) \cdot |A^-| = \sqrt{\hat{p}(1-\hat{p})} = \sigma_q$, factoring out yields:
\begin{equation}
J_{\text{GRPO}}(\theta) = \mathbb{E}_{q} \left[ 
  \sigma_q \cdot \left( \bar{\rho}^+_{\text{clip}} - \bar{\rho}^-_{\text{clip}} \right),
\right]
\label{eq:c_grpo}
\end{equation}
where the pairwise contrastive is defined as
\begin{equation}
\bar{\rho}^+_{\text{clip}} - \bar{\rho}^-_{\text{clip}}
= \frac{1}{G^+ G^-} \sum_{i=1}^{G^+} \sum_{j=1}^{G^-} \left( \mathcal{C}_{\text{up}}(\rho_i^+) - \mathcal{C}_{\text{low}}(\rho_j^-) \right).
\end{equation}
Substituting into Eq.~\ref{eq:c_grpo} gives
\begin{equation}
J_{\text{GRPO}}(\theta) = \mathbb{E}_{q} \left[ 
  \sum_{i=1}^{G^+} \sum_{j=1}^{G^-} A_{ij} \left( \mathcal{C}_{\text{up}}(\rho_i^+) - \mathcal{C}_{\text{low}}(\rho_j^-) \right)
\right]
\label{eq:pairwise_grpo}
\end{equation}
where $A_{ij} = \sigma_q / (G^+ G^-)$ is the advantage constant shared by all pairs, which can be omitted.
The pairwise form reveals that GRPO implicitly performs contrastive optimization over all positive-negative pairs. 
However, each ratio $\rho_i = \pi_\theta(o_i|q) / \pi_{\theta_{\text{old}}}(o_i|q)$ is computed with the policy conditioned only on the original query $q$, without access to other samples in the group.

\subsection{Approximating Optimal Baseline under Importance Sampling}
\label{sec:optimal_baseline}

\textbf{Log-Probability Formulation.}
While one can use $b^\ast$ in Eq.~\eqref{eq:optimal_baseline} for policy gradient update, direct computation of the importance weight $w = \pi_\theta / \pi_{\text{ref}}$ can be numerically unstable for long sequences.
We propose to use the log-probability difference as a surrogate. Let
\begin{equation}
\delta(x, y) = \log \pi_\theta(y | x) - \log \pi_{\text{ref}}(y | x),
\label{eq:delta_def}
\end{equation}
so that $w = e^\delta$. 
Under the trust region constraint, $|\delta|$ remains small, and the first-order Taylor expansion implies that $w \approx 1 + \delta$ and $w^2 \approx 1 + 2\delta$.

\textbf{Corrected Baseline.} Given such approximation, we can approximate $b^\ast$ via
\begin{equation}
b^* = \frac{\mathbb{E}[Rw^2]}{\mathbb{E}[w^2]} \approx \frac{\mathbb{E}[R] + 2\mathbb{E}[R\delta]}{1 + 2\mathbb{E}[\delta]}
\end{equation}
The trust region constraint implies $\mathbb{E}[\delta] \approx 0$, simplifying the denominator to 1. Decomposing $\mathbb{E}[R\delta] = \text{Cov}(R, \delta) + \mathbb{E}[R]\mathbb{E}[\delta] \approx \text{Cov}(R, \delta)$ yields:
\begin{equation}
b^* \approx \mathbb{E}[R] + 2 \cdot \text{Cov}(R, \delta)
\label{eq:corrected_baseline}
\end{equation}
When $\text{Cov}(R, \delta) > 0$, the correction increases the baseline, preventing high-confidence correct samples from dominating the gradient.

\section{Method}
\label{sec:method}
Section~\ref{sec:analysis} reveals that GRPO implicitly optimizes over all positive-negative pairs, yet each sample is evaluated independently without awareness of other outputs in the group. The sparse final reward further hinders its ability from updating the most important token during the training process. 
As illustrated in Figure~\ref{fig:overview}, we propose \n{} to address the limitation by conditioning each sample on opposite-partition outputs, enabling explicit contrastive learning where right and wrong attempts inform each other (\S\ref{sec:bilateral}). 
To stabilize training under bilateral conditioning, we further introduce \m{}, which adjusts the baseline using reward-confidence covariance to reduce the variance of objective (\S\ref{sec:vr_baseline}).

\subsection{Bilateral Context Conditioning}
\label{sec:bilateral}

The pairwise contrastive form in Eq.~\ref{eq:pairwise_grpo} shows that GRPO optimizes over all positive-negative pairs $(o_i^+, o_j^-)$, yet neither of the sample observes the other during policy learning.
We propose Bilateral Context Conditioning (\n{}) to enable explicit cross-partition information flow while preserving the contrastive structure. 
% Specifically, we construct augmented contexts by concatenating the query with opposite-partition samples: $x^+ = [q; \mathcal{O}^-]$ and $x^- = [q; \mathcal{O}^+]$.
% When evaluating correct solutions, the policy observes failure reasoning trace via $\mathcal{O}^-$ in $x^+$; and when evaluating incorrect ones, it observes successful traces via $\mathcal{O}^+$ in $x^-$.
Specifically, we construct augmented contexts by concatenating the query with opposite-partition samples:
\begin{equation*}
x^+ = [q; \mathcal{O}^-], \qquad x^- = [q; \mathcal{O}^+]
\label{eq:context_aug}
\end{equation*}
When evaluating correct solutions, the policy observes failure reasoning trace via $x^+$; and when evaluating incorrect ones, it observes successful traces via $x^-$.

\textbf{Conditioned Importance Sampling Ratios.}
To enable a direct information flow across partitions, we consider the following conditioned importance sampling ratio.
\begin{equation}
\begin{aligned}
\rho_i^+ &= \frac{\pi_\theta(o_i \mid x^+)}{\pi_{\theta_{\text{old}}}(o_i \mid q)}, \quad o_i \in \mathcal{O}^+ \\
\rho_j^- &= \frac{\pi_\theta(o_j \mid x^-)}{\pi_{\theta_{\text{old}}}(o_j \mid q)}, \quad o_j \in \mathcal{O}^-.
\end{aligned}
\label{eq:cond_ratio}
\end{equation}
Opposite-partition samples serve as privileged information~\citep{lopez2015unifying} available only during training, enabling contrastive learning with zero inference overhead.

\textbf{Conditioned GRPO Objective.}
Replacing the independent ratios in Eq.~\ref{eq:c_grpo} with conditioned ones yields:
%\begin{equation}
%J_{\text{cGRPO}}(\theta) = %\mathbb{E}_{q} \left[ 
%  \sum_{i=1}^{G^+} \sum_{j=1}^{G^-} A_{ij} \left( \mathcal{C}_{\text{up}}(\rho_i^+) - \mathcal{C}_{\text{low}}(\rho_j^-) \right)
%\right]
%\label{eq:cgrpo_pairwise}
%\end{equation}
%Equivalently, the objective can be written in the averaged form:
\begin{equation}
J_{\text{cGRPO}}(\theta) = \mathbb{E}_{q} \left[ 
  \sigma_q \cdot \left( \bar{\rho}^{+}_{\text{clip}} - \bar{\rho}^{-}_{\text{clip}} \right),
\right]
\label{eq:cgrpo_obj}
\end{equation}
where $\bar{\rho}^{+}_{\text{clip}} = \frac{1}{G^+} \sum_{i \in \mathcal{O}^+} \mathcal{C}_{\text{up}}(\rho_i^+)$ and $\bar{\rho}^{-}_{\text{clip}} = \frac{1}{G^-} \sum_{j \in \mathcal{O}^-} \mathcal{C}_{\text{low}}(\rho_j^-)$.
The objective structure remains identical to standard GRPO; only the conditioning of the trainable policy changes.

\textbf{Ratio Decomposition.}
It can be seen that the conditioned ratio decomposes as $\rho_i^c = w_i \cdot \rho_i$, where 
\begin{equation}
w_i = \frac{\pi_\theta(o_i \mid q, \mathcal{O}^{\mp})}{\pi_\theta(o_i \mid q)}
\label{eq:cond_weight}
\vspace{-3mm}
\end{equation}

is a conditioning weight measuring the probability shift from observing opposite-partition samples. 
When $w_i = 1$ for all samples, cGRPO reduces to standard GRPO.
The asymmetric design preserves the trust region constraint by measuring deviation from $\pi_{\theta_{\text{old}}}(\cdot|q)$. The bilateral conditioning framework can be applied uniformly across the GRPO family. As shown in Table~\ref{tab:variants}, the adaptation requires only replacing $\pi_\theta(\cdot|q)$ with $\pi_\theta(\cdot|q, \mathcal{O}^{\mp}_i)$ while preserving each method's specific design choices.

\begin{table}[t]
\centering
\caption{\n{} applied to GRPO variants. For each method, the first row shows the standard formulation and the \colorbox{blue!5}{shaded row} shows the bilateral modification. \textcolor{green!70!black}{Green} highlights the change.}
\label{tab:variants}
\scriptsize
\renewcommand{\arraystretch}{1.3}
\setlength{\tabcolsep}{3pt}
\begin{tabular}{@{}cc@{}}
\toprule
\textbf{Ratio} $\rho_i$ & \textbf{Objective} $\mathcal{L}_i$ \\
\midrule
\multicolumn{2}{@{}l@{}}{\textbf{GRPO}~\citep{guo2025deepseek}} \\[2pt]
$\dfrac{\pi_\theta(o_i|q)}{\pi_{\theta_{\text{old}}}(o_i|q)}$ 
& $\min\bigl(\rho_i A_i, \text{clip}(\rho_i, 1{\pm}\epsilon) A_i\bigr)$ \\[6pt]
\cellcolor{blue!5}$\dfrac{\pi_\theta(o_i|\textcolor{green!70!black}{q, \mathcal{O}^{\mp}_i})}{\pi_{\theta_{\text{old}}}(o_i|q)}$ 
& \cellcolor{blue!5}$\min\bigl(\rho_i A_i, \text{clip}(\rho_i, 1{\pm}\epsilon) A_i\bigr)$ \\[4pt]
\midrule
\multicolumn{2}{@{}l@{}}{\textbf{Dr.GRPO}~\citep{liu2025understanding} }\\[2pt]
$\dfrac{\pi_\theta(o_i|q)}{\pi_{\text{ref}}(o_i|q)}$ 
& $\min\bigl(\rho_i A_i, \text{clip}(\rho_i, 1{\pm}\epsilon) A_i\bigr) - \beta \mathbb{D}_{\textcolor{green!70!black}{\text{KL}}}$ \\[6pt]
\cellcolor{blue!5}$\dfrac{\pi_\theta(o_i|\textcolor{green!70!black}{q, \mathcal{O}^{\mp}_i})}{\pi_{\text{ref}}(o_i|q)}$ 
& \cellcolor{blue!5}$\min\bigl(\rho_i A_i, \text{clip}(\rho_i, 1{\pm}\epsilon) A_i\bigr) - \beta \mathbb{D}_{\textcolor{green!70!black}{\text{KL}}}$ \\[4pt]
\midrule
\multicolumn{2}{@{}l@{}}{\textbf{DAPO}~\citep{yu2025dapo} }\\[2pt]
$\dfrac{\pi_\theta(o_i|q)}{\pi_{\theta_{\text{old}}}(o_i|q)}$ 
& $\min\bigl(\rho_i A_i, \text{clip}(\rho_i, 1{-}\epsilon_l, 1{+}\epsilon_h) A_i\bigr)$ \\[6pt]
\cellcolor{blue!5}$\dfrac{\pi_\theta(o_i|\textcolor{green!70!black}{q, \mathcal{O}^{\mp}_i})}{\pi_{\theta_{\text{old}}}(o_i|q)}$ 
& \cellcolor{blue!5}$\min\bigl(\rho_i A_i, \text{clip}(\rho_i, 1{-}\epsilon_l, 1{+}\epsilon_h) A_i\bigr)$ \\[4pt]
\midrule
\multicolumn{2}{@{}l@{}}{\textbf{GSPO}~\citep{zheng2025group} }\\[2pt]
$\exp\Bigl(\dfrac{1}{|o_i|}\sum_t \log \dfrac{\pi_\theta(o_{i,t}|q)}{\pi_{\theta_{\text{old}}}(o_{i,t}|q)}\Bigr)$ 
& $\rho_i^{1/|o_i|} \cdot \text{sg}(\rho_i)^{1-1/|o_i|} \cdot A_i$ \\[6pt]
\cellcolor{blue!5}$\exp\Bigl(\dfrac{1}{|o_i|}\sum_t \log \dfrac{\pi_\theta(o_{i,t}|\textcolor{green!70!black}{q, \mathcal{O}^{\mp}_i})}{\pi_{\theta_{\text{old}}}(o_{i,t}|q)}\Bigr)$ 
& \cellcolor{blue!5}$\rho_i^{1/|o_i|} \cdot \text{sg}(\rho_i)^{1-1/|o_i|} \cdot A_i$ \\[4pt]
\bottomrule
\end{tabular}
\vspace{1mm}
\begin{flushleft}
\scriptsize
$\mathcal{O}^{\mp}_i$: opposite-partition samples for $o_i$. $\text{sg}(\cdot)$: stop-gradient. $\mathbb{D}_{\text{KL}}$: KL divergence from reference policy.
\end{flushleft}
\vspace{-3mm}
\end{table}

\subsection{Reward-Confidence Correction}
\label{sec:vr_baseline}

Standard GRPO uses the group mean reward as the baseline, which is optimal when importance weights are independent of rewards. 
However, the independence assumption rarely holds in practice since the model typically assigns higher probability to outputs it considers correct, inducing correlation between rewards and policy confidence. 
Figure~\ref{fig:cov} provides empirical evidence: $\text{Cov}(R, \delta)$ where $\delta = \log\pi_\theta - \log\pi_{\text{ref}}$ increases monotonically throughout training, reaching 0.066 for Qwen3-4B and 0.138 for Phi-4-mini. 
The distribution of $\delta$ shows clear separation between correct and incorrect samples, with mean differences of 0.27 and 0.56 respectively.

\begin{figure}[htbp]
    \centering
    \includegraphics[width=\columnwidth]{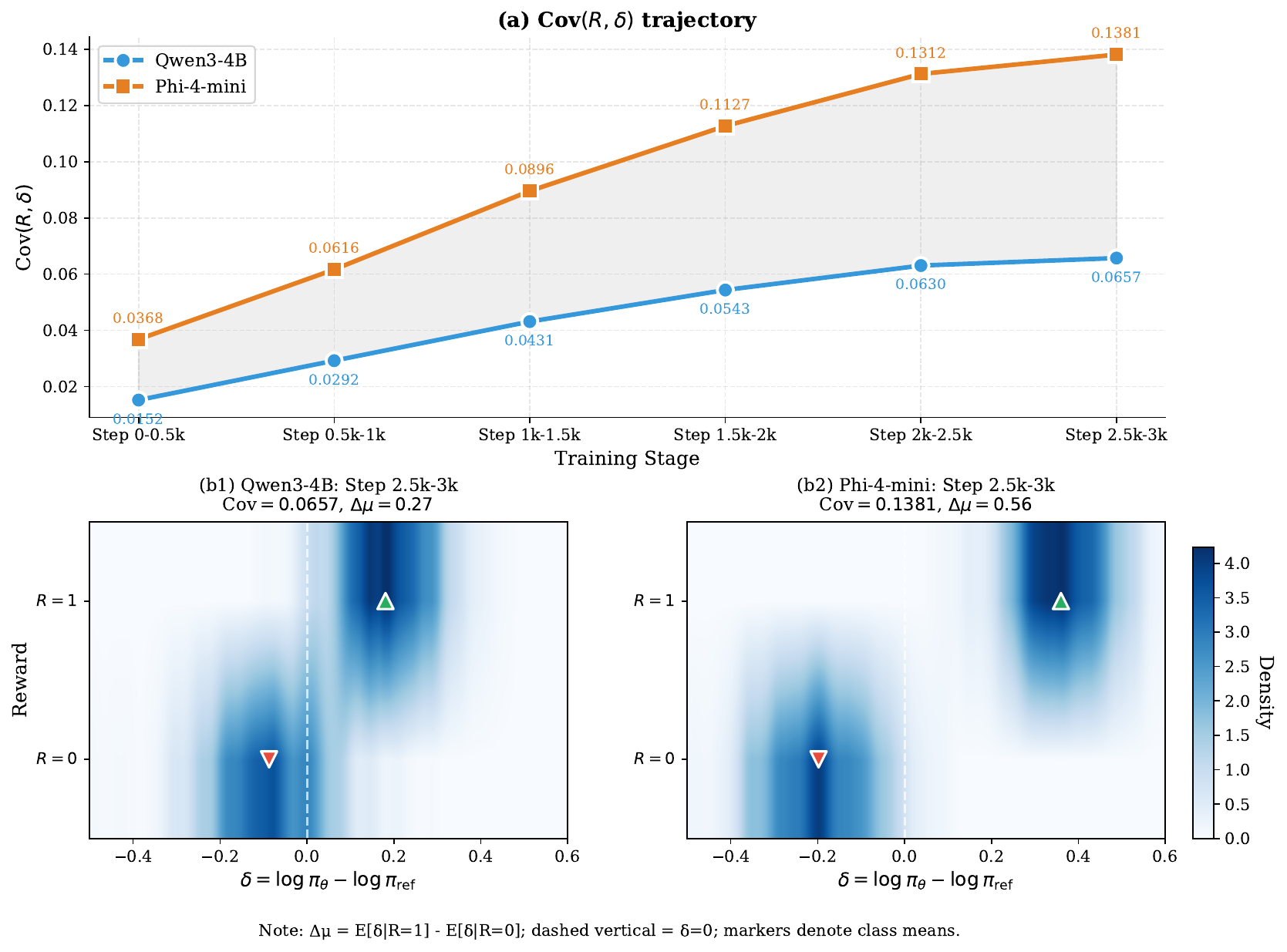}
    \caption{Reward-confidence correlation analysis. $\text{Cov}(R, \delta)$ increases throughout training for both models. Distribution of $\delta$ at training steps 2.5k--3k shows clear separation by reward, with mean separation $\Delta_\mu = 0.27$ for Qwen3-4B and $0.56$ for Phi-4-mini.}
    \label{fig:cov}
\end{figure}

\textbf{Covariance-Based Correction.}
Section~\ref{sec:optimal_baseline} derived the approximation of the variance-minimizing baseline under importance sampling:
\begin{equation}
b^* \approx \mathbb{E}[R] + 2 \cdot \text{Cov}(R, \delta)
\label{eq:optimal_baseline_method}
\end{equation}
The correction term $2 \cdot \text{Cov}(R, \delta)$ adjusts the baseline upward when the model assigns higher probability to correct outputs, which prevents high-confidence correct samples from dominating the gradient, thereby reducing variance.
We refer to the covariance-based adjustment as Reward-Confidence Correction (\m{}).

\textbf{Corrected Advantage.}
We estimate the covariance within each sampled group using sample statistics:
\begin{equation}
\widehat{\text{Cov}}(R, \delta) = \frac{1}{G} \sum_{i=1}^{G} (r_i - \bar{R})(\delta_i - \bar{\delta}),
\label{eq:cov_estimate}
\end{equation}
where $\bar{R} = \frac{1}{G}\sum_i r_i$ and $\bar{\delta} = \frac{1}{G}\sum_i \delta_i$. The corrected advantage is then:
\begin{equation}
A_i^{\text{RCC}} = r_i - \bar{R} - 2 \cdot \widehat{\text{Cov}}(R, \delta)
\label{eq:rcc_advantage}
\end{equation}

A notable departure from standard GRPO is the omission of $\sigma$-normalization. Standard GRPO computes $A_i = (r_i - \bar{R}) / \sigma$ to standardize advantages across groups with different reward variances. However, the covariance correction already provides adaptive scaling: when reward variance is high, $|\text{Cov}(R, \delta)|$ tends to be larger, naturally adjusting the correction magnitude. Empirically, combining \m{} with $\sigma$-normalization leads to over-regularization.

The computational overhead is negligible because both $\{r_i\}$ and $\{\log \pi_\theta(o_i|q)\}$ are already computed in standard GRPO, and the covariance estimate requires only $O(G)$ additional operations per query.

\begin{algorithm}[!t]
\caption{\n{} with \m{}}
\label{alg:unified}
\begin{algorithmic}[1]
\small
\REQUIRE Dataset $\mathcal{D}$, policy $\pi_\theta$, reference $\pi_{\text{ref}}$, group size $G$
\FOR{each query $q \in \mathcal{D}$}
    \STATE \textcolor{blue!70}{\textit{\# Sampling and partitioning}}
    \STATE Sample $\{o_i\}_{i=1}^G \sim \pi_{\theta_{\text{old}}}(\cdot|q)$, obtain rewards $\{r_i\}$
    \STATE $\mathcal{O}^+ \leftarrow \{o_i : r_i = 1\}$, \quad $\mathcal{O}^- \leftarrow \{o_j : r_j = 0\}$
    \STATE \textcolor{blue!70}{\textit{\# \n{}: Bilateral Context Conditioning}}
    \STATE $x^+ \leftarrow [q; \mathcal{O}^-]$, \quad $x^- \leftarrow [q; \mathcal{O}^+]$
    \FOR{$o_i \in \mathcal{O}^+ \cup \mathcal{O}^-$}
        \STATE $x^c \leftarrow x^+$ if $o_i \in \mathcal{O}^+$ else $x^-$
        \STATE $\rho_i^c \leftarrow \pi_\theta(o_i|x^c) \,/\, \pi_{\theta_{\text{old}}}(o_i|q)$
        \STATE $\delta_i \leftarrow \log \pi_\theta(o_i|x^c) - \log \pi_{\text{ref}}(o_i|q)$
    \ENDFOR
    \STATE \textcolor{blue!70}{\textit{\# \m{}: Reward-Confidence Correction}}
    \STATE $\bar{R} \leftarrow \frac{1}{G}\sum_i r_i$, \quad $\bar{\delta} \leftarrow \frac{1}{G}\sum_i \delta_i$
    \STATE $\widehat{\text{Cov}} \leftarrow \frac{1}{G}\sum_i (r_i - \bar{R})(\delta_i - \bar{\delta})$
    \STATE $A_i^{\text{RCC}} \leftarrow r_i - \bar{R} - 2 \cdot \widehat{\text{Cov}}, \quad \forall i$
    \STATE \textcolor{blue!70}{\textit{\# Policy update}}
    \STATE Update $\theta$ using $\{\rho_i^c, A_i^{\text{RCC}}\}$ with clipped surrogate objective
\ENDFOR
\STATE \textbf{return} $\pi_\theta$
\end{algorithmic}
\end{algorithm}
\vspace{-2mm}

Algorithm~\ref{alg:unified} presents the complete procedure. 
\n{} operates on $\rho$ by incorporating cross-partition information into policy evaluation, while \m{} refines the advantage by accounting for reward-confidence correlation. Operating on different components, the two mechanisms can be applied independently to any GRPO variant or combined in a single training procedure.
The key modifications to standard GRPO are: constructing bilateral contexts for cross-conditioned ratio computation and replacing the group-normalized advantage with the covariance-corrected estimate. 
At inference time, the model generates from the original prompt $q$ alone, as \n{} serves only to shape the training signal, incurring zero additional overhead.

\begin{table*}[!t]
\centering
\caption{\textbf{Main results on mathematical reasoning benchmarks.} We report the average Pass@1 accuracy (\%) over 32 independent runs using Qwen3-4B and Phi-4-mini. \textbf{\n} denotes our \textbf{Bi}lateral \textbf{C}onditioning extension applied to various GRPO-based baselines. \colorbox{rowblue}{Shaded rows} represent the proposed variants, with subscripts indicating the absolute performance change relative to the corresponding base method. \n{} yields consistent gain of 0.3--1.9 percentage points across settings, with larger gains on weaker base models. }
\label{tab:main_results}
\setlength{\tabcolsep}{4.5pt}
\renewcommand{\arraystretch}{1.2}
\footnotesize
\begin{tabular}{l cccc cccc}
\toprule
& \multicolumn{4}{c}{\textbf{Qwen3-4B-Instruct-2507}} 
& \multicolumn{4}{c}{\textbf{Phi-4-mini-instruct-3.8B}} \\
\cmidrule(lr){2-5}\cmidrule(lr){6-9}
\textbf{Method} 
& \small Math500 & \small AMC 23 & \small AIME24 & \small AIME25 
& \small Math500 & \small AMC 23 & \small AIME24 & \small AIME25 \\
\midrule
SFT Baseline
& 89.2 & 69.4 & 53.8 & 45.4
& 72.3 & 36.5 & 9.8 & 4.3 \\
\midrule
\multicolumn{9}{l}{\textit{Standard GRPO Baseline}} \\
GRPO ($G=2$)
& 90.6 & 70.2 & 53.5 & 45.8 
& 75.8 & 39.2 & 10.2 & 4.5 \\
\oursrow \textbf{\n~GRPO}
& 91.0\inc{0.4} & 71.4\inc{1.2} & 54.2\inc{0.7} & 46.2\inc{0.4}
& 76.4\inc{0.6} & 40.5\inc{1.3} & 10.5\inc{0.3} & 4.8\inc{0.3} \\
\addlinespace[3pt]
GRPO ($G=8$)
& 91.4 & 72.8 & 54.0 & 46.4
& 76.2 & 42.5 & 10.4 & 4.5 \\
\oursrow \textbf{\n~GRPO}
& 92.2\inc{0.8} & 74.2\inc{1.4} & 54.6\inc{0.6} & 47.1\inc{0.7}
& 78.1\inc{1.9} & 44.3\inc{1.8} & 11.2\inc{0.8} & 5.1\inc{0.6} \\
\midrule
\multicolumn{9}{l}{\textit{Generalization to GRPO Family ($G=8$)}} \\
Dr.GRPO~\citep{liu2025understanding}
& 91.5 & 72.9 & 53.9 & 46.5
& 76.4 & 42.6 & 10.5 & 4.6 \\
\oursrow \textbf{\n~Dr.GRPO}
& 92.1\inc{0.6} & 74.0\inc{1.1} & 54.5\inc{0.6} & 47.0\inc{0.5}
& 78.2\inc{1.8} & 44.2\inc{1.6} & 11.1\inc{0.6} & 4.5\dec{0.1} \\ 
\addlinespace[3pt]
ASPO~\citep{wang2025aspo}
& 91.6 & 72.6 & 54.1 & 46.3
& 76.5 & 42.4 & 10.5 & 4.5 \\
\oursrow \textbf{\n~ASPO}
& 91.9\inc{0.3} & 73.9\inc{1.3} & 54.7\inc{0.6} & 46.9\inc{0.6} 
& 78.3\inc{1.8} & 44.1\inc{1.7} & 11.3\inc{0.8} & 5.1\inc{0.6} \\
\addlinespace[3pt]
GMPO~\citep{zhao2025geometric}
& 91.8 & 73.5 & 54.2 & 46.5
& 76.0 & 42.1 & 10.3 & 4.4 \\
\oursrow \textbf{\n~GMPO}
& 92.5\inc{0.7} & 74.6\inc{1.1} & 54.9\inc{0.7} & 47.2\inc{0.7}
& 77.8\inc{1.8} & 42.0\dec{0.1} & 11.0\inc{0.7} & 5.0\inc{0.6} \\ 
\addlinespace[3pt]
DAPO~\citep{yu2025dapo}
& 92.5 & 74.0 & 55.2 & 47.5
& 77.5 & 43.8 & 11.5 & 4.9 \\
\oursrow \textbf{\n~DAPO}
& \textbf{93.1}\inc{0.6} & 75.2\inc{1.2} & \textbf{55.8}\inc{0.6} & 48.0\inc{0.5}
& 79.0\inc{1.5} & \textbf{45.2}\inc{1.4} & 12.1\inc{0.6} & \textbf{5.4}\inc{0.5} \\
\addlinespace[3pt]
GSPO~\citep{zheng2025group}
& 92.3 & 74.2 & 55.0 & 47.7
& 77.8 & 43.5 & 11.6 & 4.8 \\
\oursrow \textbf{\n~GSPO}
& 92.9\inc{0.6} & \textbf{75.4}\inc{1.2} & 55.5\inc{0.5} & \textbf{48.3}\inc{0.6}
& \textbf{79.2}\inc{1.4} & 44.9\inc{1.4} & \textbf{12.3}\inc{0.7} & 5.3\inc{0.5} \\
\bottomrule
\end{tabular}
\vspace{0.5em}
\begin{flushleft}
\scriptsize
\textbf{Note:} All generalization experiments utilize a fixed group size of $G=8$. \textbf{Bold} indicates the highest performance per column. Subscripts denote absolute improvement (\textcolor{teal}{positive}) or decline (\textcolor{red}{negative}) compared to the respective baseline.
\end{flushleft}
\end{table*}

\section{Result}

\begin{figure}[!t]
    \centering
    \begin{subfigure}[b]{1.0\columnwidth}
        \centering
        \includegraphics[width=\columnwidth]{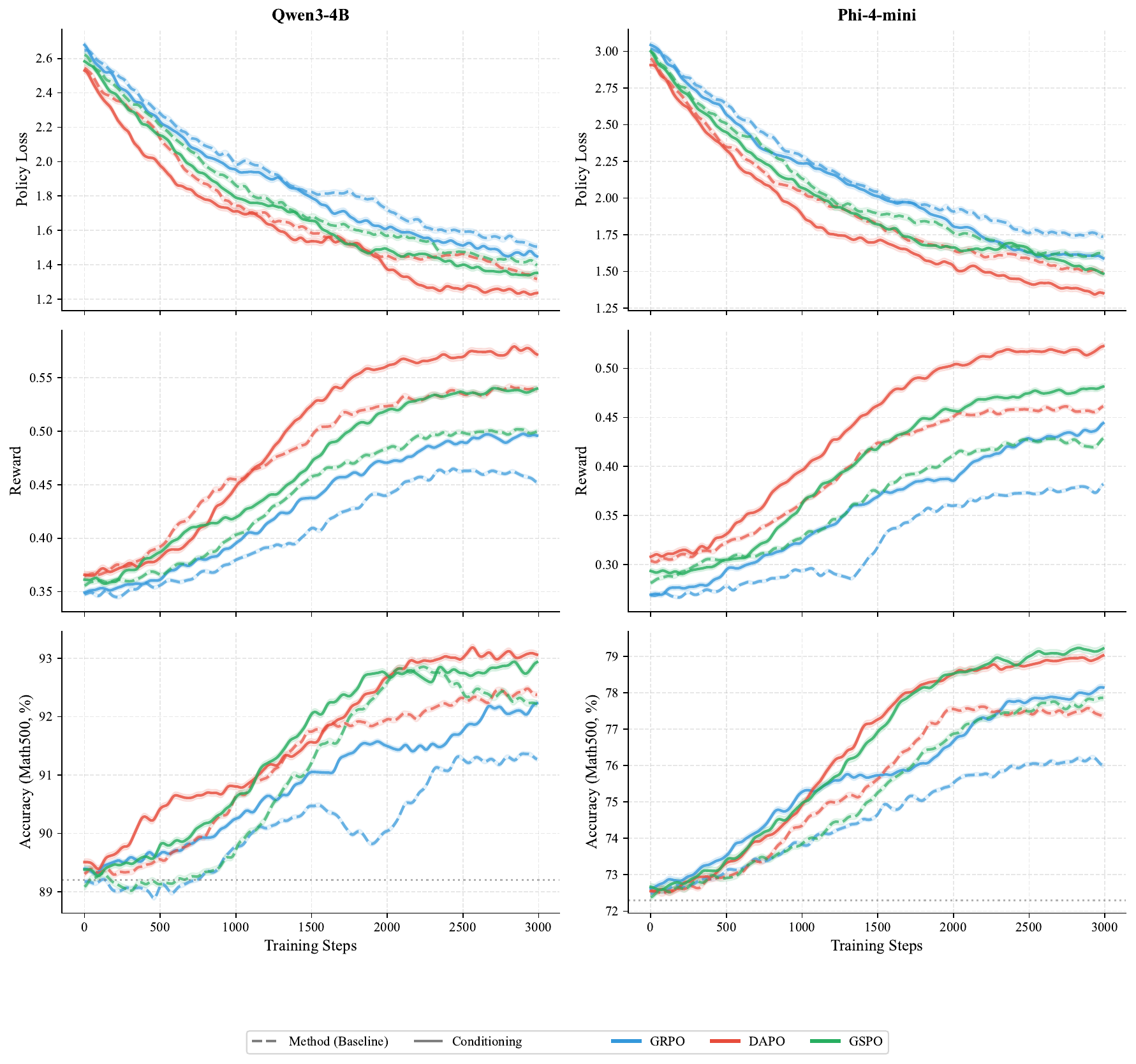}
        \vspace{-4mm}
        \caption{Training dynamics on Qwen3-4B and Phi-4-mini. Policy loss decreases steadily for all methods while average reward improves throughout training. Solid lines denote bilateral conditioning variants; dashed lines denote corresponding baselines.}
        \label{fig:train}
    \end{subfigure}
    \begin{subfigure}[b]{1.0\columnwidth}
        \centering
        \includegraphics[width=\columnwidth]{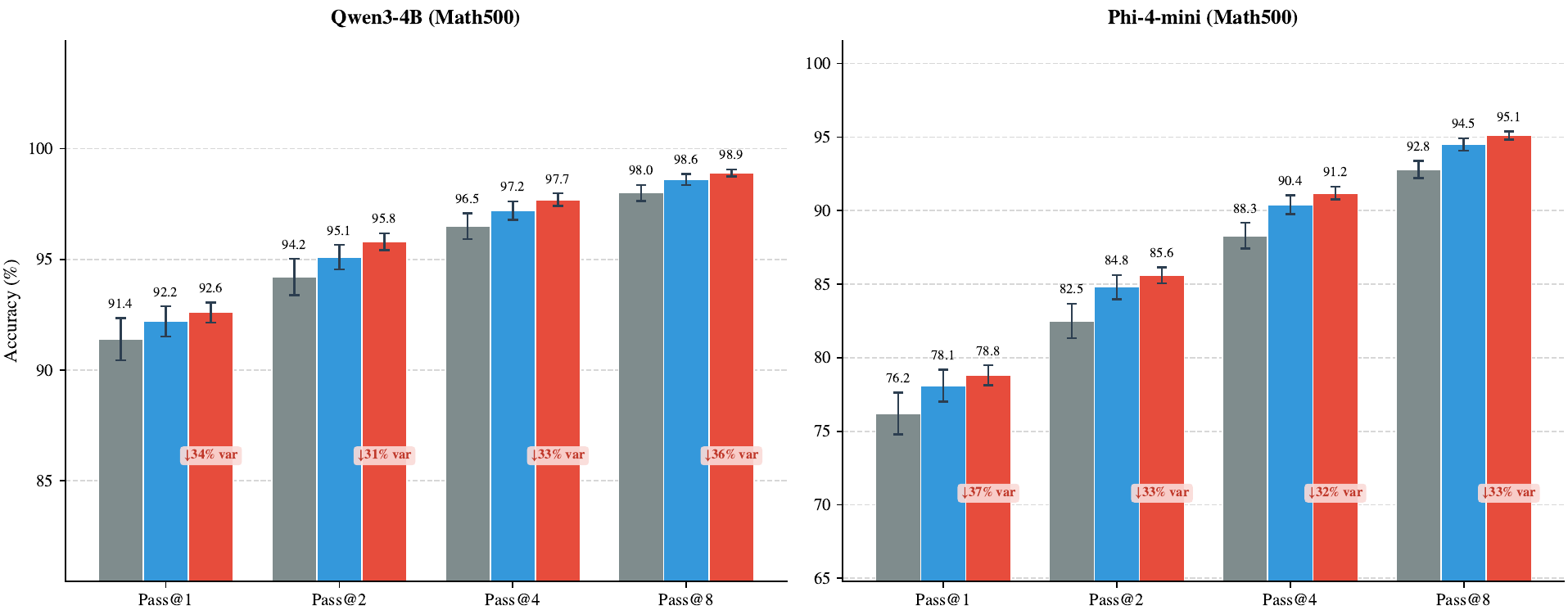}
        \vspace{-4mm}
        \caption{Pass@k accuracy on Math500 for Qwen3-4B and Phi-4-mini. \n{} denotes Bilateral Context Conditioning and \m{} denotes Reward-Confidence Correction. Error bars indicate $\pm 1$ standard deviation over 32 runs. Red annotations show the percentage of variance reduction achieved by \n{}-GRPO-\m{} relative to the GRPO baseline.}
        \label{fig:passk}
    \end{subfigure}
    \vspace{-4mm}
    \caption{Training and evaluation performance.}
    \label{fig:combined}
    \vspace{-5mm}  
\end{figure}

\subsection{Experimental Setup}

\textbf{Models and Data.}
We evaluate on two instruction-tuned models of comparable scale: Qwen3-4B-Instruct-2507 and Phi-4-mini-instruct-3.8B. 
Both models demonstrate strong baseline performance on mathematical reasoning while remaining computationally tractable for extensive experimentation.
For training, we use the DAPO-Math-17k dataset~\citep{yu2025dapo} containing approximately 17K mathematical problems, each paired with a ground-truth integer answer, formatted for a binary reward assignment based on the correctness of the answer.

\textbf{Baselines and Benchmarks.}
We apply \n{} to GRPO and recent variants, including Dr.GRPO which uses a fixed reference policy in the denominator; DAPO which employs asymmetric clipping bounds; and ASPO, GMPO, GSPO, which modify advantage estimation or probability aggregation. 
All methods share the group-based structure that our conditioning mechanism targets. 
For evaluation, we use four mathematical reasoning benchmarks of increasing difficulty: Math500~\citep{hendrycks2021measuring}; AMC 2023 from the American Mathematics Competition~\citep{cobbe2021training}; and AIME 2024/2025 from the American Invitational Mathematics Examination~\citep{aime25}.

\textbf{Implementation.}
All experiments are conducted on 4$\times$A100 GPUs. 
Models are trained with batchsize 8 and global batchsize 32, learning rate $1 \times 10^{-6}$ in AdamW optimizer. 
We use group size $G=8$ and clipping parameter $\epsilon=0.2$ unless otherwise noted. For \n{}, we concatenate opposite-partition samples to the input context; the effect of context length allocation is analyzed in Section~\ref{sec:ablation}. 
When $G^+ \neq G^-$, minority-partition samples are matched with multiple majority-partition samples. 
The presented results report Pass@1 accuracy, averaged over 32 evaluation runs with different random seeds.

\subsection{Main Results}

Table~\ref{tab:main_results} presents the comprehensive performance of \n{} applied to GRPO and its variants, demonstrating that conditioning on opposite-partition samples yields consistent improvements across all settings.
Increasing group size from $G=2$ to $G=8$ amplifies the effect: on Qwen3-4B, the improvement on Math500 grows from +0.4\% to +0.8\%, and on Phi-4-mini from +0.6\% to +1.9\%. The result aligns with the intuition that larger groups provide richer contrastive information, as the opposite-partition context $\mathcal{O}^\mp$ becomes more representative of failure or success modes. 
The improvements are also more pronounced on Phi-4-mini than on Qwen3-4B, suggesting that weaker base models benefit more from explicit contrastive signals.

The conditioning mechanism generalizes across GRPO variants. 
Whether applied to Dr.GRPO, DAPO, GMPO, or GSPO, the gains remain consistent, indicating that \n{} addresses a fundamental limitation of group-based optimization rather than artifacts of any specific algorithm. 
Among all combinations, \n{}-DAPO and \n{}-GSPO achieve the highest overall performance, reaching 93.1\% and 79.2\% on Math500 for Qwen3-4B and Phi-4-mini respectively.

Figure~\ref{fig:combined} illustrates the training behavior and evaluation performance. As shown in Figure~\ref{fig:train}, conditioned variants exhibit comparable or faster convergence relative to their baselines, with no indication of training instability despite the longer context.
Figure~\ref{fig:passk} demonstrates the effectiveness of \m{}. Building upon \n{}-GRPO, incorporating \m{} yields additional gains across all Pass@$k$ values while reducing gradient variance by 31--36\% on Qwen3-4B and 32--37\% on Phi-4-mini, demonstrating that \m{} effectively stabilizes the gradient estimation.

\subsection{Ablation and Analysis}
\label{sec:ablation}

\textbf{Context Length Allocation.}
We investigate the impact of context length allocation on opposite-partition samples by changing the proportion of the maximum context length allocated to conditional samples, and evaluate on Math500 using \n{}-GRPO with $G=8$.
Allocating too little context limits the amount of contrastive information available, while allocating too much may dilute the original query signal. 
Table~\ref{tab:ablation_context} suggests that 40\% provides a reasonable balance, though performance remains relatively stable across the range.
It is worth noting that Phi-4-mini is more sensitive to context assignment than Qwen3-4B, with a performance gap of 0.7\% between the best and worst settings, compared to 0.4\% for Qwen3-4B.
The observation suggests that models with weaker initial capabilities rely more heavily on contrastive signals, making the quality of the opposite-partition context more critical.

\begin{table}[h]
\centering
\vspace{-2mm}
\caption{Effect of context length allocation for opposite-partition samples on Math500 accuracy.}
\label{tab:ablation_context}
\small
\begin{tabular}{lcc}
\toprule
Context Ratio & Qwen3-4B & Phi-4-mini \\
\midrule
20\% & 91.8 & 77.4 \\
40\% & 92.2 & 78.1 \\
60\% & 92.0 & 77.8 \\
\bottomrule
\end{tabular}
\vspace{-3mm}
\end{table}

\textbf{Reward-Confidence Correlation.}
Section~\ref{sec:vr_baseline} established that $\text{Cov}(R, \delta)>0$ motivates \m{}, which is validated in Figure~\ref{fig:passk}.
We further examine how the correlation evolves throughout training. Figure~\ref{fig:diff} shows the distribution of $\delta = \log\pi_\theta - \log\pi_{\text{ref}}$ at different training stages, separated by reward. 
The separation between $R=0$ and $R=1$ partitions widens progressively: early in training, the distributions largely overlap, but correct and incorrect samples form distinct clusters by steps 2.5k--3k. 
Phi-4-mini exhibits larger separation ($\Delta_\mu = 0.56$) than Qwen3-4B ($\Delta_\mu = 0.27$), indicating that the weaker model develops a more pronounced confidence gap between correct and incorrect outputs. 

The analysis suggests that \m{} provides greater benefits when the base model has higher uncertainty, as the covariance correction becomes more substantial. Applying \m{} to \n{}-GRPO reduces gradient variance by approximately 25--30\% while maintaining accuracy, translating to 15--20\% faster convergence.

\begin{figure}[!t]
    \centering
    \includegraphics[width=\columnwidth]{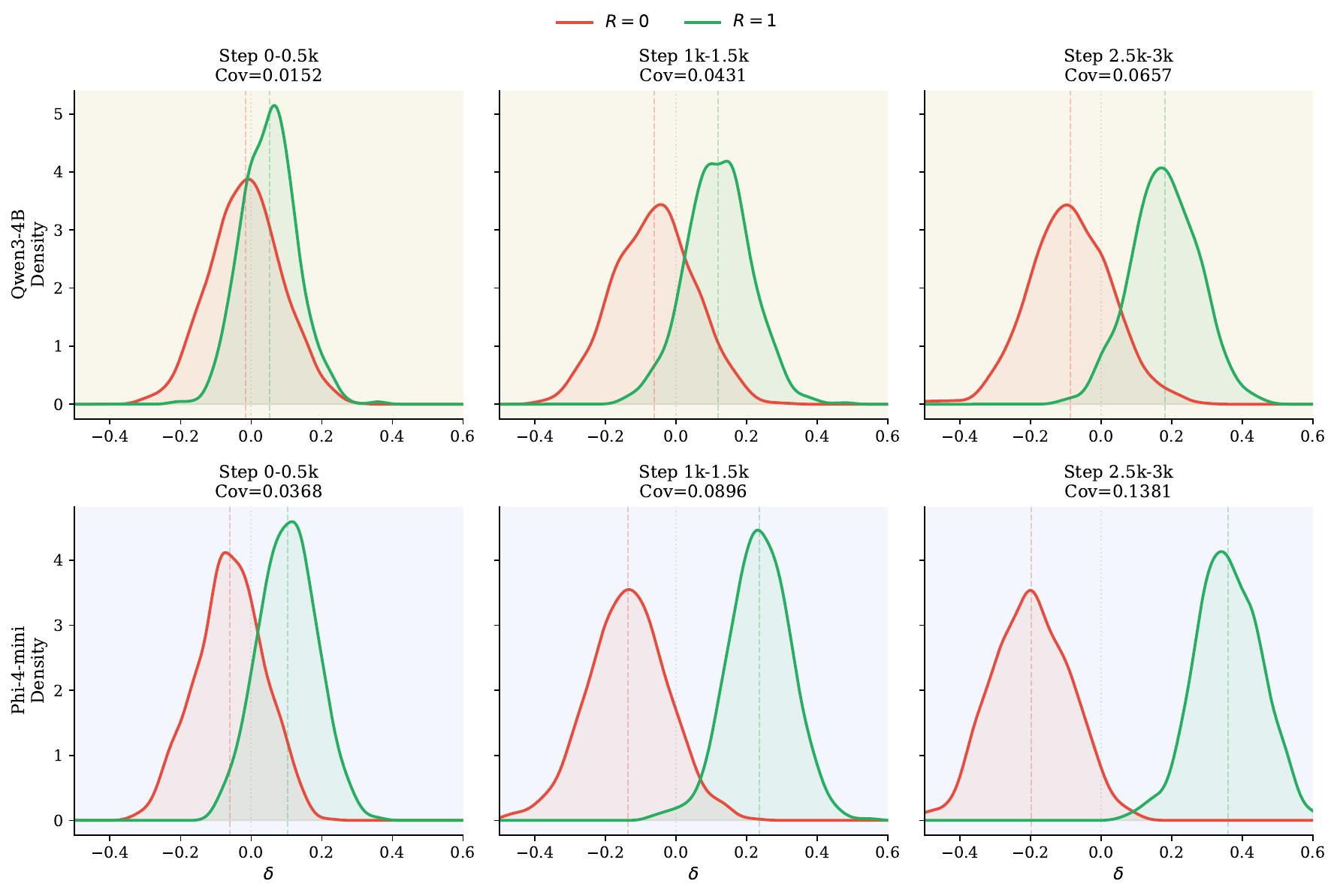}
    \vspace{-3mm}
    \caption{Evolution of $\delta$ distributions during training. The separation between $R=0$ and $R=1$ partitions widens as training progresses, reflecting the model's increasing ability to assign higher probability to correct outputs.}
    \label{fig:diff}
    \vspace{-3mm}
\end{figure}

\section{Related Work}

\textbf{Reinforcement Learning for LLM Alignment.}
The application of reinforcement learning to large language models has evolved from early RLHF approaches~\citep{christiano2017deep,ouyang2022training} to recent reasoning-focused training. Chain-of-thought prompting~\citep{wei2022chain} demonstrated that intermediate reasoning steps substantially improve performance, motivating RL methods such as those used in OpenAI's o1~\citep{jaech2024openai} and DeepSeek-R1~\citep{guo2025deepseek}. 
Contrastive approaches such as DPO~\citep{rafailov2023direct} and IPO~\citep{saeidi2025insights} frame alignment as preference optimization, bypassing reward modeling~\cite{ethayarajh2024kto,azar2024general}.
GRPO~\citep{shao2024deepseekmath} addresses PPO's computational overhead by eliminating the critic network and computing advantages through group-wise normalization. Variants include Dr.GRPO~\citep{liu2025understanding}(fixed reference policy), DAPO~\citep{yu2025dapo} (asymmetric clipping), and GSPO~\citep{zheng2025group}(geometric averaging). 
Existing methods evaluate each sample independently without leveraging contrastive signals between correct and incorrect outputs. We addresses this by conditioning on the opposite partition and integrates into any GRPO variant.

\textbf{Variance Reduction in Policy Gradients.}
Variance reduction through baselines is classical in policy gradient methods~\citep{williams1992simple}. The optimal baseline depends on the sampling distribution and generally differs from the simple reward mean~\citep{greensmith2004variance,weaver2013optimal}. 
Generalized Advantage Estimation~\citep{schulman2015high} provides a practical bias-variance tradeoff, while RLOO~\citep{ahmadian2024back} revisits REINFORCE-style optimization. 
Recent work derives theoretically optimal baselines under simplifying assumptions~\citep{hao2025policy} or replaces group-based estimation with persistent global trackers~\citep{xu2025single}.
Existing approaches assume independence between importance weights and rewards, which is increasingly violated as training progresses. Our method accounts for the resulting correlation via a covariance correction term.

\section{Conclusion}
We presented a contrastive reformulation of GRPO that reveals an unexploited structure in group-based policy optimization: within each sampled group, correct and incorrect outputs naturally partition into contrastive subsets that could inform each other during training. 
Building on the insight, we proposed Bilateral Context Conditioning, which enables cross-partition information flow where right and wrong attempts finally meet, grounded in the Learning Using Privileged Information framework. 
To stabilize training, we further introduced Reward-Confidence Correction, which exploits reward-confidence correlation for more accurate gradient estimation. 
Both mechanisms integrate seamlessly into existing GRPO variants with minimal overhead. 
Experiments on mathematical reasoning benchmarks demonstrate consistent improvements of 0.3--1.9 percentage points, with larger gains on weaker base models. 
Future work includes extending the methods to tasks with continuous rewards and other reasoning domains such as code generation.

\clearpage
\section*{Impact Statement}
This paper presents work whose goal is to advance the field of Machine
Learning. There are many potential societal consequences of our work, none
which we feel must be specifically highlighted here.

\bibliography{ref}

@article{wei2022chain,
  title={Chain-of-thought prompting elicits reasoning in large language models},
  author={Wei, Jason and Wang, Xuezhi and Schuurmans, Dale and Bosma, Maarten and Xia, Fei and Chi, Ed and Le, Quoc V and Zhou, Denny and others},
  journal={Advances in neural information processing systems},
  volume={35},
  pages={24824--24837},
  year={2022}
}

@article{ouyang2022training,
  title={Training language models to follow instructions with human feedback},
  author={Ouyang, Long and Wu, Jeffrey and Jiang, Xu and Almeida, Diogo and Wainwright, Carroll and Mishkin, Pamela and Zhang, Chong and Agarwal, Sandhini and Slama, Katarina and Ray, Alex and others},
  journal={Advances in neural information processing systems},
  volume={35},
  pages={27730--27744},
  year={2022}
}

@article{shao2024deepseekmath,
  title={Deepseekmath: Pushing the limits of mathematical reasoning in open language models},
  author={Shao, Zhihong and Wang, Peiyi and Zhu, Qihao and Xu, Runxin and Song, Junxiao and Bi, Xiao and Zhang, Haowei and Zhang, Mingchuan and Li, YK and Wu, Yang and others},
  journal={arXiv preprint arXiv:2402.03300},
  year={2024}
}

@article{guo2025deepseek,
  title={Deepseek-r1: Incentivizing reasoning capability in llms via reinforcement learning},
  author={Guo, Daya and Yang, Dejian and Zhang, Haowei and Song, Junxiao and Zhang, Ruoyu and Xu, Runxin and Zhu, Qihao and Ma, Shirong and Wang, Peiyi and Bi, Xiao and others},
  journal={arXiv preprint arXiv:2501.12948},
  year={2025}
}

@article{schulman2017proximal,
  title={Proximal policy optimization algorithms},
  author={Schulman, John and Wolski, Filip and Dhariwal, Prafulla and Radford, Alec and Klimov, Oleg},
  journal={arXiv preprint arXiv:1707.06347},
  year={2017}
}

@article{williams1992simple,
  title={Simple statistical gradient-following algorithms for connectionist reinforcement learning},
  author={Williams, Ronald J},
  journal={Machine learning},
  volume={8},
  number={3},
  pages={229--256},
  year={1992},
  publisher={Springer}
}

@article{sutton1999policy,
  title={Policy gradient methods for reinforcement learning with function approximation},
  author={Sutton, Richard S and McAllester, David and Singh, Satinder and Mansour, Yishay},
  journal={Advances in neural information processing systems},
  volume={12},
  year={1999}
}

@inproceedings{schulman2015trust,
  title={Trust region policy optimization},
  author={Schulman, John and Levine, Sergey and Abbeel, Pieter and Jordan, Michael and Moritz, Philipp},
  booktitle={International conference on machine learning},
  pages={1889--1897},
  year={2015},
  organization={PMLR}
}

@inproceedings{kakade2002approximately,
  title={Approximately optimal approximate reinforcement learning},
  author={Kakade, Sham and Langford, John},
  booktitle={Proceedings of the nineteenth international conference on machine learning},
  pages={267--274},
  year={2002}
}

@article{schulman2015high,
  title={High-dimensional continuous control using generalized advantage estimation},
  author={Schulman, John and Moritz, Philipp and Levine, Sergey and Jordan, Michael and Abbeel, Pieter},
  journal={arXiv preprint arXiv:1506.02438},
  year={2015}
}

@article{greensmith2004variance,
  title={Variance reduction techniques for gradient estimates in reinforcement learning},
  author={Greensmith, Evan and Bartlett, Peter L and Baxter, Jonathan},
  journal={Journal of Machine Learning Research},
  volume={5},
  number={Nov},
  pages={1471--1530},
  year={2004}
}

@article{mroueh2025revisiting,
  title={Revisiting Group Relative Policy Optimization: Insights into On-Policy and Off-Policy Training},
  author={Mroueh, Youssef and Dupuis, Nicolas and Belgodere, Brian and Nitsure, Apoorva and Rigotti, Mattia and Greenewald, Kristjan and Navratil, Jiri and Ross, Jerret and Rios, Jesus},
  journal={arXiv preprint arXiv:2505.22257},
  year={2025}
}

@article{li2025inspo,
  title={Inspo: Unlocking intrinsic self-reflection for llm preference optimization},
  author={Li, Yu and Lan, Tian and Qi, Zhengling},
  journal={arXiv preprint arXiv:2512.23126},
  year={2025}
}

@article{christiano2017deep,
  title={Deep reinforcement learning from human preferences},
  author={Christiano, Paul F and Leike, Jan and Brown, Tom and Martic, Miljan and Legg, Shane and Amodei, Dario},
  journal={Advances in neural information processing systems},
  volume={30},
  year={2017}
}

@article{jaech2024openai,
  title={Openai o1 system card},
  author={Jaech, Aaron and Kalai, Adam and Lerer, Adam and Richardson, Adam and El-Kishky, Ahmed and Low, Aiden and Helyar, Alec and Madry, Aleksander and Beutel, Alex and Carney, Alex and others},
  journal={arXiv preprint arXiv:2412.16720},
  year={2024}
}

@article{rafailov2023direct,
  title={Direct preference optimization: Your language model is secretly a reward model},
  author={Rafailov, Rafael and Sharma, Archit and Mitchell, Eric and Manning, Christopher D and Ermon, Stefano and Finn, Chelsea},
  journal={Advances in neural information processing systems},
  volume={36},
  pages={53728--53741},
  year={2023}
}

@inproceedings{azar2024general,
  title={A general theoretical paradigm to understand learning from human preferences},
  author={Azar, Mohammad Gheshlaghi and Guo, Zhaohan Daniel and Piot, Bilal and Munos, Remi and Rowland, Mark and Valko, Michal and Calandriello, Daniele},
  booktitle={International Conference on Artificial Intelligence and Statistics},
  pages={4447--4455},
  year={2024},
  organization={PMLR}
}

@article{ethayarajh2024kto,
  title={Kto: Model alignment as prospect theoretic optimization},
  author={Ethayarajh, Kawin and Xu, Winnie and Muennighoff, Niklas and Jurafsky, Dan and Kiela, Douwe},
  journal={arXiv preprint arXiv:2402.01306},
  year={2024}
}

@article{lopez2015unifying,
  title={Unifying distillation and privileged information},
  author={Lopez-Paz, David and Bottou, L{\'e}on and Sch{\"o}lkopf, Bernhard and Vapnik, Vladimir},
  journal={arXiv preprint arXiv:1511.03643},
  year={2015}
}

@article{ahmadian2024back,
  title={Back to basics: Revisiting reinforce style optimization for learning from human feedback in llms},
  author={Ahmadian, Arash and Cremer, Chris and Gall{\'e}, Matthias and Fadaee, Marzieh and Kreutzer, Julia and Pietquin, Olivier and {\"U}st{\"u}n, Ahmet and Hooker, Sara},
  journal={arXiv preprint arXiv:2402.14740},
  year={2024}
}

@article{yu2025dapo,
  title={Dapo: An open-source llm reinforcement learning system at scale},
  author={Yu, Qiying and Zhang, Zheng and Zhu, Ruofei and Yuan, Yufeng and Zuo, Xiaochen and Yue, Yu and Dai, Weinan and Fan, Tiantian and Liu, Gaohong and Liu, Lingjun and others},
  journal={arXiv preprint arXiv:2503.14476},
  year={2025}
}

@article{weaver2013optimal,
  title={The optimal reward baseline for gradient-based reinforcement learning},
  author={Weaver, Lex and Tao, Nigel},
  journal={arXiv preprint arXiv:1301.2315},
  year={2013}
}

@article{hendrycks2021measuring,
  title={Measuring mathematical problem solving with the math dataset},
  author={Hendrycks, Dan and Burns, Collin and Kadavath, Saurav and Arora, Akul and Basart, Steven and Tang, Eric and Song, Dawn and Steinhardt, Jacob},
  journal={arXiv preprint arXiv:2103.03874},
  year={2021}
}

@article{cobbe2021training,
  title={Training verifiers to solve math word problems},
  author={Cobbe, Karl and Kosaraju, Vineet and Bavarian, Mohammad and Chen, Mark and Jun, Heewoo and Kaiser, Lukasz and Plappert, Matthias and Tworek, Jerry and Hilton, Jacob and Nakano, Reiichiro and others},
  journal={arXiv preprint arXiv:2110.14168},
  year={2021}
}

@inproceedings{pechyony2010theory,
  title={On the theory of learning with Privileged Information},
  author={Dmitry Pechyony and Vladimir Naumovich Vapnik},
  booktitle={Neural Information Processing Systems},
  year={2010},
  url={https://api.semanticscholar.org/CorpusID:13217651}
}

@inproceedings{lightman2023let,
  title={Let's verify step by step},
  author={Lightman, Hunter and Kosaraju, Vineet and Burda, Yuri and Edwards, Harrison and Baker, Bowen and Lee, Teddy and Leike, Jan and Schulman, John and Sutskever, Ilya and Cobbe, Karl},
  booktitle={The Twelfth International Conference on Learning Representations},
  year={2023}
}

@misc{aime25,
      title={American Invitational Mathematics Examination (AIME) 2025}, 
      author={Zhang, Yifan and Math-AI, Team},
      year={2025},
}

@article{liu2025understanding,
  title={Understanding r1-zero-like training: A critical perspective},
  author={Liu, Zichen and Chen, Changyu and Li, Wenjun and Qi, Penghui and Pang, Tianyu and Du, Chao and Lee, Wee Sun and Lin, Min},
  journal={arXiv preprint arXiv:2503.20783},
  year={2025}
}

@article{wang2025aspo,
  title={Aspo: Asymmetric importance sampling policy optimization},
  author={Wang, Jiakang and Liu, Runze and Lin, Lei and Hu, Wenping and Li, Xiu and Zhang, Fuzheng and Zhou, Guorui and Gai, Kun},
  journal={arXiv preprint arXiv:2510.06062},
  year={2025}
}

@article{zhao2025geometric,
  title={Geometric-mean policy optimization},
  author={Zhao, Yuzhong and Liu, Yue and Liu, Junpeng and Chen, Jingye and Wu, Xun and Hao, Yaru and Lv, Tengchao and Huang, Shaohan and Cui, Lei and Ye, Qixiang and others},
  journal={arXiv preprint arXiv:2507.20673},
  year={2025}
}

@article{zheng2025group,
  title={Group sequence policy optimization},
  author={Zheng, Chujie and Liu, Shixuan and Li, Mingze and Chen, Xiong-Hui and Yu, Bowen and Gao, Chang and Dang, Kai and Liu, Yuqiong and Men, Rui and Yang, An and others},
  journal={arXiv preprint arXiv:2507.18071},
  year={2025}
}

@article{yang2025qwen3,
  title={Qwen3 technical report},
  author={Yang, An and Li, Anfeng and Yang, Baosong and Zhang, Beichen and Hui, Binyuan and Zheng, Bo and Yu, Bowen and Gao, Chang and Huang, Chengen and Lv, Chenxu and others},
  journal={arXiv preprint arXiv:2505.09388},
  year={2025}
}

@inproceedings{tyen2024llms,
  title={LLMs cannot find reasoning errors, but can correct them given the error location},
  author={Tyen, Gladys and Mansoor, Hassan and C{\u{a}}rbune, Victor and Chen, Yuanzhu Peter and Mak, Tony},
  booktitle={Findings of the Association for Computational Linguistics: ACL 2024},
  pages={13894--13908},
  year={2024}
}

@article{uesato2022solving,
  title={Solving math word problems with process-and outcome-based feedback},
  author={Uesato, Jonathan and Kushman, Nate and Kumar, Ramana and Song, Francis and Siegel, Noah and Wang, Lisa and Creswell, Antonia and Irving, Geoffrey and Higgins, Irina},
  journal={arXiv preprint arXiv:2211.14275},
  year={2022}
}

@article{yue2025vapo,
  title={Vapo: Efficient and reliable reinforcement learning for advanced reasoning tasks},
  author={Yue, Yu and Yuan, Yufeng and Yu, Qiying and Zuo, Xiaochen and Zhu, Ruofei and Xu, Wenyuan and Chen, Jiaze and Wang, Chengyi and Fan, TianTian and Du, Zhengyin and others},
  journal={arXiv preprint arXiv:2504.05118},
  year={2025}
}

@article{hao2025policy,
  title={On-Policy RL with Optimal Reward Baseline},
  author={Hao, Yaru and Dong, Li and Wu, Xun and Huang, Shaohan and Chi, Zewen and Wei, Furu},
  journal={arXiv preprint arXiv:2505.23585},
  year={2025}
}

@article{xu2025single,
  title={Single-stream policy optimization},
  author={Xu, Zhongwen and Ding, Zihan},
  journal={arXiv preprint arXiv:2509.13232},
  year={2025}
}

@inproceedings{saeidi2025insights,
  title={Insights into alignment: Evaluating dpo and its variants across multiple tasks},
  author={Saeidi, Amir and Verma, Shivanshu and Uddin, Md Nayem and Baral, Chitta},
  booktitle={Proceedings of the 63rd Annual Meeting of the Association for Computational Linguistics (Volume 4: Student Research Workshop)},
  pages={409--421},
  year={2025}
}

@article{mroueh2025reinforcement,
  title={Reinforcement Learning with Verifiable Rewards: GRPO's Effective Loss, Dynamics, and Success Amplification},
  author={Mroueh, Youssef},
  journal={arXiv preprint arXiv:2503.06639},
  year={2025}
}
\bibliographystyle{icml2026}

\appendix
\onecolumn

\section{Theoretical Derivations}
\subsection{Optimal Baseline Derivation}\label{sec: optimal baseline derivation}

We consider the objective of maximizing the expected reward for reasoning tasks, defined as $J(\theta) = \mathbb{E}_{q \sim \mathcal{D},\, o \sim \pi_\theta(\cdot|q)}[R(o, q)]$, where $R(o, q) \in \{0, 1\}$. To reduce the variance of the  objective estimator for a specific group $i$, we introduce a baseline $b_i$. Minimizing the variance of the estimator is equivalent to minimizing its second moment with respect to $b_i$.

\begin{proof}
Let $w_j$ denote the importance weight and $R_j$ the reward for the $j$-th sample. We define the loss function $L(b_i)$ as the expected squared weighted advantage:
\begin{equation}
    L(b_i) = \mathbb{E} \left[ \left( (R_j - b_i) w_j \right)^2 \right]
\end{equation}
Expanding the quadratic term and applying the linearity of expectation, we obtain:
\begin{align}
    L(b_i) &= \mathbb{E} \left[ R_j^2 w_j^2 - 2 R_j w_j^2 b_i + b_i^2 w_j^2 \right] \nonumber \\
           &= \mathbb{E}[R_j^2 w_j^2] - 2b_i \mathbb{E}[R_j w_j^2] + b_i^2 \mathbb{E}[w_j^2]
\end{align}
To find the optimal baseline $b_i^*$, we take the derivative of $L(b_i)$ with respect to $b_i$ and set it to zero:
\begin{equation}
    \frac{\partial L}{\partial b_i} = -2 \mathbb{E} [ R_j w_j^2 ] + 2b_i \mathbb{E} [ w_j^2 ] = 0
\end{equation}
Rearranging the terms yields:
\begin{equation}
    b_i \mathbb{E} [ w_j^2 ] = \mathbb{E} [ R_j w_j^2 ]
\end{equation}
Solving for $b_i$, we derive the optimal baseline:
\begin{equation}
    b_i^* = \frac{\mathbb{E} [ R_j w_j^2 ]}{\mathbb{E} [ w_j^2 ]}
\end{equation}
\end{proof}

\subsection{Contrastive Reformulation of GRPO}
\label{app:contrastive}

\subsubsection{Advantage Simplification under Binary Rewards}

\begin{proposition}
Under binary rewards $r_i \in \{0, 1\}$, the group-normalized advantages reduce to:
\begin{equation}
A^+ = \sqrt{\frac{1-\hat{p}}{\hat{p}}}, \quad A^- = -\sqrt{\frac{\hat{p}}{1-\hat{p}}}
\end{equation}
where $\hat{p} = G^+/G$ is the proportion of correct samples.
\end{proposition}

\begin{proof}
From Eq.~5, the advantage is $A_i = (r_i - \mu)/\sigma$ with $\mu = \frac{1}{G}\sum_j r_j$ and $\sigma^2 = \frac{1}{G}\sum_j(r_j - \mu)^2$.

Under binary rewards, $\mu = G^+/G = \hat{p}$. The variance becomes:
\begin{align}
\sigma^2 &= \frac{1}{G}\left[G^+ (1-\hat{p})^2 + G^- \hat{p}^2\right] = \hat{p}(1-\hat{p})^2 + (1-\hat{p})\hat{p}^2 = \hat{p}(1-\hat{p})
\end{align}

For $r_i = 1$: $A^+ = (1 - \hat{p})/\sqrt{\hat{p}(1-\hat{p})} = \sqrt{(1-\hat{p})/\hat{p}}$.

For $r_i = 0$: $A^- = -\hat{p}/\sqrt{\hat{p}(1-\hat{p})} = -\sqrt{\hat{p}/(1-\hat{p})}$.
\end{proof}

\begin{corollary}
\label{cor:advantage_magnitude}
The advantages satisfy $\hat{p} \cdot A^+ = (1-\hat{p}) \cdot |A^-| = \sqrt{\hat{p}(1-\hat{p})} = \sigma_q$.
\end{corollary}

\subsubsection{Asymmetric Clipping Behavior}

\begin{lemma}
\label{lemma:clipping}
The clipped objective $L_i^{\text{CLIP}} = \min\left(\rho_i A_i, \text{clip}(\rho_i, 1-\epsilon, 1+\epsilon) A_i\right)$ satisfies:
\begin{equation}
L_i^{\text{CLIP}} = \begin{cases}
A_i \cdot \mathcal{C}_{\text{up}}(\rho_i) & \text{if } A_i > 0 \\
A_i \cdot \mathcal{C}_{\text{low}}(\rho_i) & \text{if } A_i < 0
\end{cases}
\end{equation}
where $\mathcal{C}_{\text{up}}(\rho) = \min(\rho, 1+\epsilon)$ and $\mathcal{C}_{\text{low}}(\rho) = \max(\rho, 1-\epsilon)$.
\end{lemma}

\begin{proof}
When $A_i > 0$, both terms in the $\min$ are positive. The clipping only activates when $\rho_i > 1+\epsilon$, capping the contribution at $(1+\epsilon)A_i$. Thus $L_i^{\text{CLIP}} = A_i \cdot \min(\rho_i, 1+\epsilon)$.

When $A_i < 0$, both terms are negative, so $\min$ selects the more negative value. The clipping activates when $\rho_i < 1-\epsilon$, yielding $L_i^{\text{CLIP}} = A_i \cdot \max(\rho_i, 1-\epsilon)$.
\end{proof}

\subsubsection{Derivation of the Contrastive Form}

\begin{theorem}
\label{thm:contrastive}
Under binary rewards, the GRPO objective can be written as:
\begin{equation}
\mathcal{J}_{\text{GRPO}}(\theta) = \mathbb{E}_q\left[\sigma_q \cdot \left(\bar{\rho}^+_{\text{clip}} - \bar{\rho}^-_{\text{clip}}\right)\right]
\end{equation}
where $\bar{\rho}^+_{\text{clip}} = \frac{1}{G^+}\sum_{i \in \mathcal{O}^+} \mathcal{C}_{\text{up}}(\rho_i)$ and $\bar{\rho}^-_{\text{clip}} = \frac{1}{G^-}\sum_{j \in \mathcal{O}^-} \mathcal{C}_{\text{low}}(\rho_j)$.
\end{theorem}

\begin{proof}
Separating the GRPO objective by partition and applying Lemma~\ref{lemma:clipping}:
\begin{align}
\mathcal{J}_{\text{GRPO}}(\theta) &= \mathbb{E}_q\left[\frac{1}{G}\left(\sum_{i \in \mathcal{O}^+} A^+ \mathcal{C}_{\text{up}}(\rho_i) + \sum_{j \in \mathcal{O}^-} A^- \mathcal{C}_{\text{low}}(\rho_j)\right)\right] \\
&= \mathbb{E}_q\left[\hat{p} \cdot A^+ \cdot \bar{\rho}^+_{\text{clip}} + (1-\hat{p}) \cdot A^- \cdot \bar{\rho}^-_{\text{clip}}\right]
\end{align}

By Corollary~\ref{cor:advantage_magnitude}, $\hat{p} \cdot A^+ = \sigma_q$ and $(1-\hat{p}) \cdot A^- = -\sigma_q$:
\begin{equation}
\mathcal{J}_{\text{GRPO}}(\theta) = \mathbb{E}_q\left[\sigma_q \cdot \bar{\rho}^+_{\text{clip}} - \sigma_q \cdot \bar{\rho}^-_{\text{clip}}\right] = \mathbb{E}_q\left[\sigma_q \cdot \left(\bar{\rho}^+_{\text{clip}} - \bar{\rho}^-_{\text{clip}}\right)\right]
\end{equation}
\end{proof}

\subsubsection{Pairwise Contrastive Form}

\begin{corollary}
The objective admits a pairwise formulation:
\begin{equation}
\mathcal{J}_{\text{GRPO}}(\theta) = \mathbb{E}_q\left[\sum_{i=1}^{G^+}\sum_{j=1}^{G^-} A_{ij} \left(\mathcal{C}_{\text{up}}(\rho_i^+) - \mathcal{C}_{\text{low}}(\rho_j^-)\right)\right]
\end{equation}
where $A_{ij} = \sigma_q / (G^+ G^-)$.
\end{corollary}

\begin{proof}
For any sets $\{a_i\}_{i=1}^m$ and $\{b_j\}_{j=1}^n$:
\begin{equation}
\frac{1}{m}\sum_{i=1}^m a_i - \frac{1}{n}\sum_{j=1}^n b_j = \frac{1}{mn}\sum_{i=1}^m \sum_{j=1}^n (a_i - b_j)
\end{equation}

Applying this identity to $\bar{\rho}^+_{\text{clip}} - \bar{\rho}^-_{\text{clip}}$ and multiplying by $\sigma_q$ yields the result.
\end{proof}

The pairwise form shows GRPO implicitly optimizes over all $(o_i^+, o_j^-)$ pairs, yet each ratio $\rho_i = \pi_\theta(o_i|q)/\pi_{\theta_{\text{old}}}(o_i|q)$ conditions only on query $q$, without observing other samples.

%==============================================================================
\subsection{Optimal Baseline under Importance Sampling}
\label{app:baseline}

\subsubsection{Variance-Minimizing Baseline}

\begin{theorem}
\label{thm:optimal_baseline}
The baseline minimizing variance of the importance-weighted gradient estimator is:
\begin{equation}
b^* = \frac{\mathbb{E}[R w^2]}{\mathbb{E}[w^2]}
\end{equation}
where $w = \pi_\theta / \pi_{\text{ref}}$ is the importance weight.
\end{theorem}

\begin{proof}
The gradient estimator is $\hat{g} = (R - b) \cdot w \cdot \psi$ where $\psi = \nabla_\theta \log \pi_\theta$. Since subtracting a baseline does not change the expected gradient, minimizing variance reduces to minimizing $f(b) = \mathbb{E}[(R-b)^2 w^2 \psi^2]$.

Setting $\frac{df}{db} = -2\mathbb{E}[(R-b) w^2 \psi^2] = 0$ gives:
\begin{equation}
b^* = \frac{\mathbb{E}[R w^2 \psi^2]}{\mathbb{E}[w^2 \psi^2]}
\end{equation}

Under the standard approximation that $\psi^2$ is approximately independent of $R$ and $w$ \citep{greensmith2004variance}, this simplifies to $b^* = \mathbb{E}[R w^2]/\mathbb{E}[w^2]$.
\end{proof}

When $w \equiv 1$ (on-policy), this reduces to $b^* = \mathbb{E}[R]$.

\subsubsection{First-Order Approximation under Trust Region}

\begin{proposition}
\label{prop:covariance}
Under the trust region constraint, the optimal baseline admits:
\begin{equation}
b^* \approx \mathbb{E}[R] + 2 \cdot \text{Cov}(R, \delta)
\end{equation}
where $\delta = \log \pi_\theta(y) - \log \pi_{\text{ref}}(y)$.
\end{proposition}

\begin{proof}
With $w = e^\delta$, Taylor expansion under small $|\delta|$ gives $w \approx 1 + \delta$ and $w^2 \approx 1 + 2\delta$. Thus:
\begin{align}
\mathbb{E}[Rw^2] &\approx \mathbb{E}[R] + 2\mathbb{E}[R\delta] \\
\mathbb{E}[w^2] &\approx 1 + 2\mathbb{E}[\delta]
\end{align}

The trust region constraint implies $\mathbb{E}[\delta] \approx 0$, so $\mathbb{E}[w^2] \approx 1$. Decomposing $\mathbb{E}[R\delta] = \text{Cov}(R, \delta) + \mathbb{E}[R]\mathbb{E}[\delta] \approx \text{Cov}(R, \delta)$ yields:
\begin{equation}
b^* \approx \mathbb{E}[R] + 2 \cdot \text{Cov}(R, \delta)
\end{equation}
\end{proof}

\subsubsection{Corrected Advantage}

The corrected advantage for sample $i$ is:
\begin{equation}
A_i^{\text{RCC}} = r_i - \bar{R} - 2 \cdot \widehat{\text{Cov}}(R, \delta)
\end{equation}
where $\bar{R} = \frac{1}{G}\sum_j r_j$ and $\widehat{\text{Cov}}(R, \delta) = \frac{1}{G}\sum_j (r_j - \bar{R})(\delta_j - \bar{\delta})$.

When $\text{Cov}(R, \delta) > 0$ (model assigns higher probability to correct outputs), the correction increases the baseline, preventing high-confidence correct samples from dominating the gradient. When the correlation is zero, RCC reduces to standard GRPO.

We omit $\sigma$-normalization in RCC because the covariance term already provides adaptive scaling: when reward variance is high, $|\text{Cov}(R, \delta)|$ tends to be larger. Empirically, combining both normalizations leads to over-regularization.

\subsubsection{Computational Overhead}

RCC requires $O(G)$ additional operations per query: computing $\delta_i = \log \pi_\theta(o_i|q) - \log \pi_{\text{ref}}(o_i|q)$ (already available from the importance ratio), then $\bar{R}$, $\bar{\delta}$, and the sample covariance. This is negligible compared to the forward pass.

\section{Experimental Details}
\label{app:experimental}

\subsection{Hyperparameters}
\label{app:hyperparameters}

Tables~\ref{tab:hyperparameters} and \ref{tab:variant_hyperparameters} list the hyperparameter configurations.

\begin{table}[h]
\centering
\caption{Hyperparameter configuration (shared across both models unless noted).}
\label{tab:hyperparameters}
\begin{tabular}{ll}
\toprule
\textbf{Hyperparameter} & \textbf{Value} \\
\midrule
\multicolumn{2}{l}{\textit{Optimization}} \\
Learning rate & $1 \times 10^{-6}$ \\
Optimizer (Adam $\beta_1, \beta_2$) & AdamW (0.9, 0.999) \\
Weight decay / Gradient clipping & 0.01 / 1.0 \\
LR schedule (warmup / total steps) & Cosine (100 / 3000) \\
\midrule
\multicolumn{2}{l}{\textit{Batch \& Model}} \\
Per-device batch size / GPUs / Global batch & 8 / 4 / 32 \\
Precision / Flash Attention & bfloat16 / v2 \\
\midrule
\multicolumn{2}{l}{\textit{GRPO \& BICC}} \\
Group size $G$ / Clipping $\epsilon$ & 8 / 0.2 \\
Context allocation ratio / Max tokens & 0.4 / 2048 \\
\midrule
\multicolumn{2}{l}{\textit{Generation (training / evaluation)}} \\
Max new tokens & 2048 / 4096 \\
Temperature / Top-$p$ & 1.0, 1.0 / 0.6, 0.95 \\
\bottomrule
\end{tabular}
\end{table}

\begin{table}[h]
\centering
\caption{Variant-specific hyperparameters.}
\label{tab:variant_hyperparameters}
\begin{tabular}{lll}
\toprule
\textbf{Method} & \textbf{Parameter} & \textbf{Value} \\
\midrule
Dr.GRPO & KL coefficient $\beta$; Reference policy & 0.01; Frozen initial \\
DAPO & Clip bounds $(\epsilon_l, \epsilon_h)$ & (0.2, 0.28) \\
GMPO & Geometric mean exponent & $1/T_i$ \\
GSPO & Aggregation; Stop-gradient & Geometric mean; Yes \\
\bottomrule
\end{tabular}
\end{table}

\subsection{Dataset and Evaluation}
\label{app:dataset}

\subsubsection{Training Data}

The DAPO-Math-17k dataset \citep{yu2025dapo} contains $\sim$17,000 mathematical problems with integer answers. The reward function assigns binary scores: $R(o, a^*) = \mathbf{1}[\text{extract}(o) = a^*]$, where $\text{extract}(\cdot)$ parses the final answer (matching ``The answer is X'' or ``$\boxed{X}$'') and $a^*$ is the ground truth.

\subsubsection{Evaluation Benchmarks and Protocol}

\begin{table}[h]
\centering
\caption{Evaluation benchmarks and sampling configuration.}
\label{tab:eval_all}
\begin{tabular}{lcccc}
\toprule
\textbf{Benchmark} & \textbf{Problems} & \textbf{Difficulty} & \textbf{Answer Type} & \textbf{Source} \\
\midrule
Math500 & 500 & Competition & Integer/Expr & \citet{hendrycks2021measuring} \\
AMC 2023 & 40 & Competition & Integer & AMC 10/12 \\
AIME 2024 & 30 & Olympiad & Integer (0--999) & AIME I/II \\
AIME 2025 & 30 & Olympiad & Integer (0--999) & AIME I/II \\
\midrule
\multicolumn{5}{l}{\textit{Evaluation Configuration}: Temperature 0.6, Top-$p$ 0.95, Max tokens 4096, Seeds 0--31} \\
\bottomrule
\end{tabular}
\end{table}

For Pass@k metrics (Figure~3b), we generate $n=32$ samples per problem and use the unbiased estimator \citep{chen2021evaluating}: $\text{Pass@}k = \mathbb{E}[1 - \binom{n-c}{k}/\binom{n}{k}]$ where $c$ is the number of correct samples. We report 95\% confidence intervals via bootstrap resampling. For pairwise comparisons, Wilcoxon signed-rank tests confirm all improvements on Math500 and AMC 2023 are significant at $p < 0.001$; AIME results remain significant at $p < 0.05$.

\subsubsection{Computational Overhead}

BICC introduces additional computational cost from two sources. First, constructing augmented contexts requires concatenating opposite-partition samples to the original query, increasing the input sequence length. The overhead scales approximately linearly with the context allocation ratio: at 40\% allocation, input sequences are roughly 1.4$\times$ longer on average. Second, computing conditioned log-probabilities $\log \pi_\theta(o_i|q, \mathcal{O}^\mp)$ requires forward passes with these extended inputs, increasing both memory consumption and computation time proportionally to sequence length.

The memory overhead primarily affects activation storage during the forward pass, as model parameters and optimizer states remain unchanged. Activation checkpointing can mitigate this at the cost of additional recomputation during the backward pass. In practice, the overall training time increase depends on hardware utilization and batch size adjustments.

RCC introduces negligible overhead. The correction requires computing $\delta_i = \log \pi_\theta(o_i|q) - \log \pi_{\text{ref}}(o_i|q)$ and the sample covariance $\widehat{\text{Cov}}(R, \delta)$, both using quantities already available in standard GRPO. 
The additional operations including $O(G)$ subtractions, multiplications, and summations per query are insignificant compared to the forward and backward pass costs.

When either partition is empty ($G^+ = 0$ or $G^- = 0$), bilateral conditioning cannot be applied. In these cases, BICC falls back to standard GRPO, incurring no additional cost. The frequency of such cases depends on problem difficulty and model capability: very easy problems often yield all-correct groups, while very hard problems yield all-incorrect groups.

At inference time, BICC incurs zero overhead since the bilateral context is used only during training to shape the learning signal. The trained model generates from the original prompt $q$ alone, identical to standard GRPO.

\section{Supplementary Experiments}
\label{app:experiments}

\subsection{RCC Ablation}
\label{app:rcc_ablation}

Table~\ref{tab:rcc_ablation} separates the contributions of BICC and RCC on Math500 with $G$=8, complementing the Pass@k analysis in Figure~3b.

\begin{table}[h]
\centering
\caption{Ablation of BICC and RCC on Math500 (\%).}
\label{tab:rcc_ablation}
\begin{tabular}{lcc}
\toprule
\textbf{Configuration} & \textbf{Qwen3-4B} & \textbf{Phi-4-mini} \\
\midrule
GRPO (baseline) & 91.4 & 76.2 \\
BICC-GRPO & 92.2 & 78.1 \\
BICC-GRPO + RCC & 92.6 & 78.8 \\
\bottomrule
\end{tabular}
\end{table}

BICC contributes the primary accuracy gain (+0.8\% Qwen, +1.9\% Phi). RCC provides additional improvement (+0.4\% Qwen, +0.7\% Phi) while reducing gradient variance by 31--37\% as shown in Figure~3b.

\subsection{Reward-Confidence Correlation}
\label{app:cov_evolution}

Table~\ref{tab:cov_values} provides numerical values for the $\text{Cov}(R, \delta)$ evolution visualized in Figure~2.

\begin{table}[h]
\centering
\caption{$\text{Cov}(R, \delta)$ across training stages.}
\label{tab:cov_values}
\begin{tabular}{lcccccc}
\toprule
\textbf{Model} & 0--0.5k & 0.5k--1k & 1k--1.5k & 1.5k--2k & 2k--2.5k & 2.5k--3k \\
\midrule
Qwen3-4B & 0.015 & 0.029 & 0.043 & 0.054 & 0.063 & 0.066 \\
Phi-4-mini & 0.037 & 0.062 & 0.090 & 0.113 & 0.131 & 0.138 \\
\bottomrule
\end{tabular}
\end{table}

Phi-4-mini exhibits stronger reward-confidence correlation throughout training (0.138 vs 0.066 at convergence), explaining the larger RCC benefit observed for this model.

\section{Qualitative Analysis and Discussion}
\label{app:discussion}

\subsection{Case Study}
\label{app:case_study}

We illustrate the effect of bilateral conditioning through a representative example. Consider the problem: \textit{Find the smallest positive integer $n$ such that $n^2 + 2n + 2$ is divisible by $7$.} Without bilateral conditioning, the model generates diverse attempts independently. Correct solutions typically identify $n \equiv 2 \pmod{7}$ through systematic modular arithmetic, while incorrect attempts often test small values without systematic search, misapply the quadratic formula in modular context, or make arithmetic errors during substitution.

Under BICC, evaluating a correct solution involves observing incorrect attempts as context. This exposure to failure modes appears to strengthen commitment to systematic approaches—the model more consistently applies modular arithmetic rather than trial-and-error. Conversely, when evaluating an incorrect solution, observing correct attempts highlights where reasoning diverges, effectively increasing the gradient penalty for specific error patterns. The conditioning weight $w_i = \pi_\theta(o_i|q, \mathcal{O}^\mp)/\pi_\theta(o_i|q)$ captures this quantitatively: $w_i > 1$ for correct solutions indicates that observing failures increases confidence in correct approaches, while $w_i < 1$ for incorrect solutions indicates that observing successes decreases confidence in flawed approaches.

\subsection{Failure Modes}
\label{app:failure_modes}

BICC does not universally improve performance. When all sampled outputs are correct ($G^- = 0$) or all incorrect ($G^+ = 0$), bilateral conditioning cannot be applied, and the method falls back to standard GRPO. This occurs more frequently on very easy or very hard problems. Additionally, when incorrect solutions share the same fundamental error (e.g., all misread the problem), the opposite-partition context provides redundant signal; BICC benefits most when incorrect attempts exhibit diverse failure modes. For problems requiring very long solutions, the opposite-partition context may be heavily truncated due to length constraints, limiting the information available for conditioning.

\subsection{Why Weaker Models Benefit More}
\label{app:weak_model}

The experimental results consistently show larger BICC improvements on Phi-4-mini compared to Qwen3-4B. We hypothesize three contributing factors. First, weaker models exhibit greater variance in solution quality within each group, providing richer contrastive signal—when correct and incorrect solutions are more clearly separated, the bilateral context carries more information. Second, Figure~2 shows Phi-4-mini develops larger $\text{Cov}(R, \delta)$ during training (0.138 vs 0.066), indicating stronger reward-confidence correlation that makes RCC more impactful. Third, with lower baseline accuracy, there is simply more room for improvement; the relative improvement may be similar across models even when absolute gains differ.

\subsection{Connection to Related Frameworks}
\label{app:connections}

BICC shares motivation with self-consistency decoding \citep{wang2023selfconsistency}, which aggregates multiple samples at inference time. BICC instead leverages sample diversity during training, enabling the model to internalize contrastive patterns without inference-time overhead. The pairwise reformulation (Eq.~15) also reveals structural similarity to InfoNCE-style objectives; unlike standard contrastive learning where negatives are constructed or retrieved, BICC uses naturally occurring incorrect attempts from the same query, ensuring task-relevant contrast. The observation that harder problems benefit more suggests BICC implicitly implements a form of curriculum learning, where the model learns to distinguish subtle reasoning errors.

\subsection{Limitations and Future Directions}
\label{app:limitations}

Our theoretical analysis and experiments assume binary rewards. Extending to continuous rewards requires soft partition assignments, complicating the contrastive structure—a natural formulation would use reward magnitude for partition membership weights. BICC increases training cost due to longer input sequences; while the accuracy-efficiency trade-off is favorable in our experiments, this overhead may be prohibitive for very large models. We evaluate exclusively on mathematical reasoning with verifiable answers; generalization to open-ended tasks where correctness is not binary remains unexplored. Finally, BICC effectiveness depends on reward signal quality—noisy or misspecified rewards could lead to misleading contrastive signals.

Several directions merit further investigation: learning to selectively weight conditioning samples based on contrastive value, conditioning on solutions from related problems to enable cross-query transfer, and integrating BICC with process reward models for step-level contrastive learning.

\end{document}